\documentclass[letterpaper, 10pt, conference]{ieeeconf}      

\IEEEoverridecommandlockouts                              

\usepackage{xcolor}
\usepackage{comment}

\usepackage{cite}

\usepackage[mathcal]{euscript}

\usepackage[cmex10]{amsmath}
\usepackage{amssymb}

\usepackage{amsthm} 
  
\usepackage{dsfont} 
\usepackage{mathabx} 

\usepackage{epsfig} 
\usepackage{pgfplots}
\usetikzlibrary{arrows.meta}
\usepackage{multirow}
\usepackage{enumerate}

\usepackage[ruled, vlined, linesnumbered]{algorithm2e}

\newtheoremstyle{plain}
	  {}
	  {}
	  {\itshape}
	  {}
	  {\bfseries}
	  {}
	  {5pt plus 1pt minus 1pt}
	  {}

\newtheoremstyle{definition}
  	  {}
	  {}
	  {\normalfont}
	  {}
	  {\bfseries}
	  {}
	  {5pt plus 1pt minus 1pt}
	  {}
  
\theoremstyle{plain}

\newtheorem{lemma}{Lemma}
\newtheorem{proposition}{Proposition}

\theoremstyle{definition}
\newtheorem{definition}{Definition}

\newcommand{\refeq}[1]			{(\ref{#1})} 
\newcommand{\reffig}[1]			{Fig. \ref{#1}} 
\newcommand{\refsec}[1]			{Section \ref{#1}}
\newcommand{\refapp}[1]			{Appendix \ref{#1}}
\newcommand{\reftab}[1]			{Table \ref{#1}}

\newcommand{\refprop}[1]		{Proposition \ref{#1}}

\newcommand{\reflem}[1]			{Lemma \ref{#1}}
\newcommand{\refdef}[1]			{Definition \ref{#1}}

\newcommand{\reffn}[1] 		    {\textsuperscript{\ref{#1}}}

\theoremstyle{plain}

\newcommand{\R}  	{\mathbb{R}} 
\newcommand{\Z}  	{\mathbb{Z}} 

\newcommand{\radius} 	{\rho}
\newcommand{\ctr} 	{\vect{c}}

\newcommand{\conv}      {\mathrm{conv}} 
\newcommand{\cone}      {\mathrm{C}} 
\newcommand{\dshape}{\mathrm{D}} 
\newcommand{\fwdsim}{\mathrm{F}} 
\newcommand{\ovect}[1] {\begin{bmatrix}\cos #1\\ \sin #1 \end{bmatrix}}
\newcommand{\ovectsmall}[1] {\scalebox{0.85}{$\ovect{#1}$}}
\newcommand{\ovecT}[1] {\tr{\ovect{#1}\!}}
\newcommand{\ovecTsmall}[1] {\scalebox{0.85}{$\ovecT{#1}$}}
\newcommand{\nvect}[1] {\begin{bmatrix}-\sin #1\\ \cos #1 \end{bmatrix}}

\newcommand{\nvecT}[1] {\tr{\nvect{#1}\!}}
\newcommand{\nvecTsmall}[1] {\scalebox{0.85}{$\nvecT{#1}$}}

\newcommand{\Rmat}{\mat{R}}

\newcommand{\headingline} 		{H}

\newcommand{\motionset}{\mathcal{M}} 

\newcommand{\freespace}	{\mathcal{F}} 
\newcommand{\workspace}	{\mathcal{W}} 
\newcommand{\ball}      {\mathrm{B}} 
\newcommand{\obstspace}	{\mathcal{O}} 

\newcommand{\path}      {\mathrm{p}}
\newcommand{\pathparam} {s}
\newcommand{\smin}      {\pathparam_{\min}}
\newcommand{\smax}      {\pathparam_{\max}}
\newcommand{\startpos}  {\pos_{\mathrm{start}}}
\newcommand{\goalpos}   {\pos_{\mathrm{goal}}}

\newcommand{\pos} 		{\vect{x}} 			
\newcommand{\ort}	    {\theta}			
\newcommand{\headingerror}{\psi}            
\newcommand{\fwdheadingerror}{\overrightarrow{\psi}}
\newcommand{\bckheadingerror}{\overleftarrow{\psi}}
\newcommand{\totalturning}{\Theta}
\newcommand{\fwdtotalturning}{\overrightarrow{\Theta}}
\newcommand{\bcktotalturning}{\overleftarrow{\Theta}}

\newcommand{\linvel}     {v}  
\newcommand{\fwdlinvel}     {\overrightarrow{v}}  
\newcommand{\bcklinvel}     {\overleftarrow{v}}  
\newcommand{\angvel}     {\omega} 
\newcommand{\fwdangvel}     {\overrightarrow{\omega}} 
\newcommand{\bckangvel}     {\overleftarrow{\omega}} 
\newcommand{\Si}{\mathrm{Si}} 
\newcommand{\xpos} {\widehat{\pos}} 
\newcommand{\xposr} {\widehat{\pos}^{r}} 
\newcommand{\cpos}	{\widebar{\pos}} 
\newcommand{\cposr}	{\widebar{\pos}^{r}} 
\newcommand{\fort} 		{\theta^{*}} 

\newcommand{\lingain}   	{\kappa_{v}} 			
\newcommand{\anggain}   	{\kappa_{\omega}} 			

\newcommand{\goal}		{\vect{x^{*}}}

\newcommand{\gain}   	{\kappa} 			
\newcommand{\ctrl}      {\vect{u}} 			
\newcommand{\fwdctrl}      {\overrightarrow{\vect{u}}} 			
\newcommand{\bckctrl}      {\overleftarrow{\vect{u}}} 			

\newcommand{\safedist}  {\mathrm{dist}}
\newcommand{\safelevel}{\sigma} 

\let\originalleft\left
\let\originalright\right
\renewcommand{\left}{\mathopen{}\mathclose\bgroup\originalleft}
\renewcommand{\right}{\aftergroup\egroup\originalright}
	
\newcommand{\plist}[1] 	{\left(#1\right)} 
\newcommand{\blist}[1]	{\left[ #1 \right]} 
\newcommand{\clist}[1]	{\left\{#1\right\}} 

\DeclareMathOperator{\arctantwo}{arctan2}

\newcommand{\vect}[1]   {\mathrm{#1}}
\newcommand{\mat}[1]    {\mathbf{#1}}
\newcommand{\tr}[1] {{#1}^{\mathrm{T}}} 
\newcommand{\norm}[1]  {\|#1\|}
\newcommand{\absval}[1]{\left|#1 \right|} 

\newcommand{\ldf}   {:=} 

\newcommand{\diff} {\mathrm{d}}
\DeclareMathOperator{\si}{Si} 

\title{\LARGE \bf 
Total Turning and Motion Range Prediction for Safe Unicycle Control \\[0mm] {\large (Technical Report)}}

\author{Abdulla Tarshahani  and Aykut \.{I}\c{s}leyen and \"{O}m\"{u}r Arslan 
\thanks{The authors are with the Department of Mechanical Engineering, Eindhoven University of Technology, P.O. Box 513, 5600 MB Eindhoven, The Netherlands. The authors are also affiliated with the Eindhoven AI Systems Institute. Emails: a.tarshahani@student.tue.nl, \{a.isleyen, o.arslan\}@tue.nl}
}

\begin{document}

\maketitle
\thispagestyle{empty}
\pagestyle{empty}

\begin{abstract}
Safe and smooth motion control is essential for mobile robots when performing various automation tasks around obstacles, especially in the presence of people and other mobile robots.
The total turning and space used by a mobile robot while moving towards a specified goal position play a crucial role in determining the required control effort and complexity.   
In this paper, we consider a standard unicycle control approach based on angular feedback linearization and provide an explicit analytical measure for determining the total turning effort during unicycle control in terms of unicycle state and control gains.
We show that undesired  spiral oscillatory motion around the goal position can be avoided by choosing a higher angular control gain compared to the linear control gain. 
Accordingly, we establish an accurate, explicit triangular motion range bound on the closed-loop unicycle trajectory using the total turning effort.   
The improved accuracy in motion range prediction results from a stronger dependency on the unicycle state and control parameters.
To compare alternative circular, conic, and triangular motion range prediction approaches, we present an application of the proposed unicycle motion control and motion prediction methods for safe unicycle path following around obstacles in numerical simulations.
\end{abstract}

\section{Introduction}
\label{sec.Introduction}

Autonomous mobile robots are key enablers for flexible automation in many various applications settings, including logistics \cite{renan_nascimento_RAS2021, fiorini_botturi_ISR2008} and service industries \cite{kim_etal_RAM2009, jones_RAM2006}. 
Safe and smooth autonomous motion around obstacles is crucial for mobile robots to perform automation tasks in complex environments, including interaction with people  and other mobile robots \cite{gul_rahiman_alhady_sahal_CE2019, snape_etal_IROS2010, philippsen_siegwart_ICRA2003}. 
Accurate motion prediction plays a key role in safety assessment, planning, and controlling autonomous robot motion around obstacles \cite{chakravarthy_debasish_TSM1998, fox_burgard_thrun_RAM1997, fiorini_shiller_IJRR1998, arslan_koditschek_ICRA2017, isleyen_vandewouw_arslan_RAL2022, arslan_isleyen_arXiv2023, li_yi_niu_atanasov_arXiv2023}.

In this paper, we consider a standard unicycle control approach using angular feedback linearization and introduce an explicit, accurate triangular motion bound on the resulting closed-loop unicycle trajectory. 
This prediction is based on an analytical estimation of the total turning effort and final orientation of the unicycle control. 
The improved accuracy of the triangular motion prediction, compared to alternative circular and conic motion predictions (illustrated in \reffig{fig.Unicycle_Feedback_Motion_Prediction}), results from its stronger dependence on unicycle state and control gains. 
We apply these unicycle motion prediction methods for safe path-following control around obstacles.

\begin{figure}[t]
\centering
\begin{tabular}{@{}c@{\hspace{0.5mm}}c@{\hspace{0.5mm}}c@{\hspace{1mm}}c@{}}
\includegraphics[width = 0.33\columnwidth]{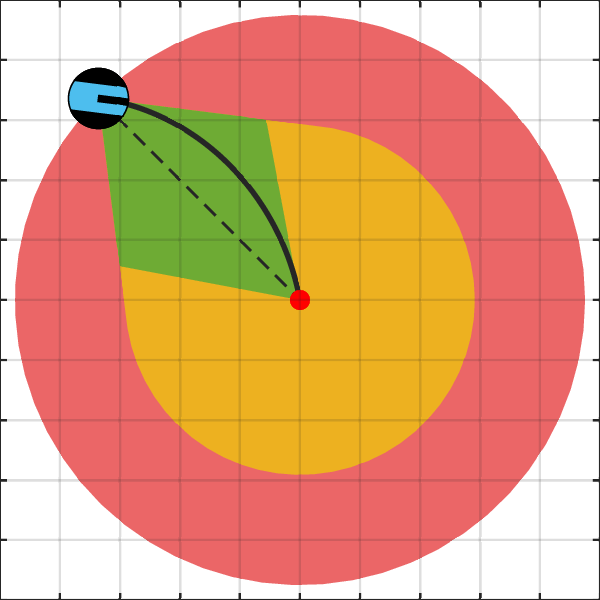} & 
\includegraphics[width = 0.33\columnwidth]{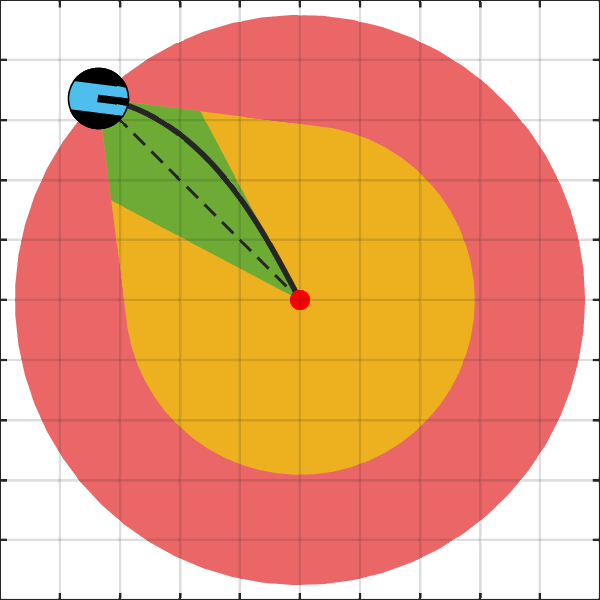} &
\includegraphics[width = 0.33\columnwidth]{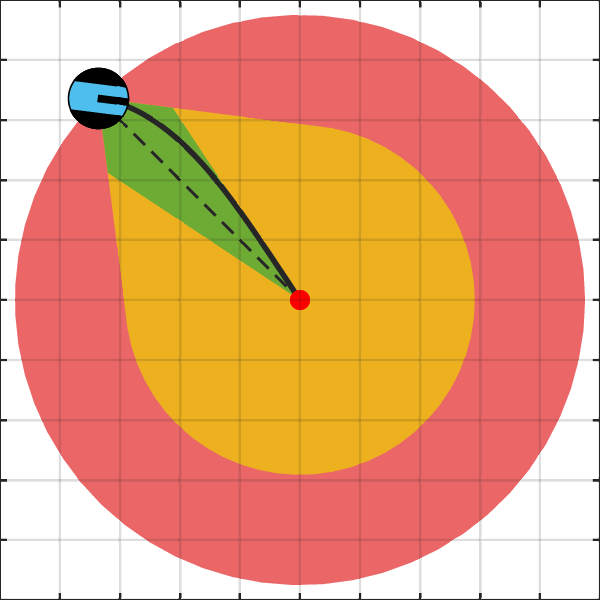} 
\end{tabular}
\vspace{-3.5mm}
\caption{Unicycle feedback motion predictions that bound the closed-loop unicycle motion trajectory (black line) towards a given goal  position (red dot) using circular (red), conic (orange), and triangular (green) motion sets for a shared linear control gain of $\lingain= 1$ and various angular control gains (left)   $\anggain = 1 $,   (middle) $\anggain = 2$  and (right) $\anggain = 3 $.
The triangular motion prediction varies as the control parameters change, mainly due to its stronger dependency on the control gains.}
\label{fig.Unicycle_Feedback_Motion_Prediction}
\vspace{-4.5mm}
\end{figure}

\subsection{Motivation and Relevant Literature}

Designing safe and smooth autonomous robot motion requires systematic understanding and characterization  of closed-loop robot motion under a feedback motion controller.
Existing control approaches for unicycle mobile robots primarily focus on the stability and convergence of closed-loop robot motion \cite{astolfi_JDSMC1999, astolfi_SCL1996, lee_etal_IROS2000, deluca_oriolo_vandittelli_IFAC2000, das_etal_TRA2002, novel_campion_bastin_IJRR1995}, but pay little attention to the geometric motion characteristics that are crucial for safety \cite{isleyen_vandewouw_arslan_IROS2023, isleyen_vandewouw_arslan_CDC2023}.
In our earlier work \cite{isleyen_vandewouw_arslan_IROS2023}, we present a family of conic feedback motion range bounds for a standard inner-outer-loop unicycle control approach \cite{astolfi_JDSMC1999} as a more accurate alternative to the standard circular Lyapunov sublevel sets.
This improvement is mainly due to the observation that, in addition to the straight-line Euclidean distance to the goal, the orientational goal alignment distance decreases during the closed-loop unicycle motion.   
In a follow-up work \cite{isleyen_vandewouw_arslan_CDC2023}, we introduce a new unicycle adaptive headway motion control approach based on feedback linearization with a headway point.
We demonstrate that under this adaptive headway control, the closed-loop unicycle motion can be accurately bounded by a triangular region defined by the convex hull of the unicycle position, the goal position, and the headway point.
This stronger dependency on the unicycle state and the control parameter (i.e., the headway point) allows for a simpler and more accurate motion range bound for the adaptive headway control.   
In this present paper, we aim to bridge the gap between the unicycle control and motion prediction methods in our previous works \cite{isleyen_vandewouw_arslan_IROS2023} and \cite{isleyen_vandewouw_arslan_CDC2023}.
To achieve this, we explore another standard inner-outer-loop unicycle motion control approach based on angular feedback linearization \cite{lee_etal_IROS2000}.
We provide an explicit measure to determine both the total turning effort and the final orientation during the unicycle control. 	 
Using the knowledge of the final unicycle orientation, we build a new accurate triangular motion range bound that surpasses alternative conic and circular motion range bounds due to its stronger dependence on the unicycle state and the control gains.

Predicting the future motion of autonomous systems is essential for ensuring safety, control, and planning of mobile robots navigating around obstacles \cite{lefevre_vasquez_laugier_ROBOMECH2014}.
Feedback motion prediction for finding a bounding motion set on the closed-loop motion trajectory of a mobile robot moving under a known control policy, allows for informative safety assessment and effective control and planning strategies around obstacles \cite{arslan_arXiv2022, isleyen_vandewouw_arslan_RAL2022, arslan_isleyen_arXiv2023,arslan_koditschek_ICRA2017}. 
Reachability analysis provides numerical methods for estimating such motion bounds for a wide range of control systems \cite{althoff_dolan_TR02014, althoff_frehse_girard_ARCRAS2021}.
However, it often involves high computational costs, making it less suitable for real-time motion planning and control, and lacks intuitive understanding and explicit characterization of closed-loop system motion.
For globally convergent autonomous systems, the concept of forward and backward reachable sets \cite{mitchell_HSCC2007} is straightforward. 
This is because the forward reachable set corresponds to the closed-loop system trajectory, given the autonomous nature of the system. 
Meanwhile, the backward reachability set encompasses the entire state space, thanks to global convergence.
In robotics, open-loop motion prediction based on forward system simulation or high-level motion planning using simple physical motion models (e.g., constant velocity, acceleration, and turning rates) \cite{schubert_richter_wanielik_ICIF2008}, or predefined/learned motion patterns (termed as motion primitives and maneuvers) \cite{schreier_willert_adamy_TITS2016, bennewitz_etal_IJRR2005}, also finds significant applications. However, such open-loop motion prediction methods are often not suitable for verifiable safety assessment, planning, and control.
In this paper, we propose new analytic (circular, conic, and triangular) motion prediction methods to bound the closed-loop unicycle robot motion by exploiting the geometric characteristic properties of the unicycle control. We apply these feedback motion prediction methods for verifiably safe unicycle path following around obstacles and compare their performance with the numerically computed forward reachable set of the closed-loop unicycle dynamics.

\subsection{Contributions and Organization of the Paper}

This paper introduces new explicit model-based methods for determining total turning effort and feedback motion prediction for safe unicycle control around obstacles.
In \refsec{sec.Unicycle_Dynamics_Control}, we briefly summarize a standard globally convergent control approach for the kinematic unicycle model using angular feedback linearization.
In \refsec{sec.Unicycle_Motion_Prediction}, we introduce an explicit way of determining the total turning effort and the final unicycle orientation under this unicycle control approach.
Through a systematic analysis of closed-loop unicycle motion, we design a highly accurate analytical triangular feedback motion prediction method that outperforms  circular and conic alternatives. 
In \refsec{sec.Safe Unicycle_Path_Following}, to compare these feedback motion prediction methods, we provide an example application of these unicycle feedback motion prediction methods for safe path following around obstacles in numerical simulations.
We conclude in \refsec{sec.Conclusions} with a summary of our contributions and future directions.

\section{Unicycle Dynamics \& Control}
\label{sec.Unicycle_Dynamics_Control}

In this section, we provide a brief description of the kinematic unicycle robot model and present a standard unicycle control approach using angular feedback linearization. 
We then highlight important geometric properties of the closed-loop unicycle motion to build an intuitive understanding and  characterize the resulting motion patterns.     

\subsection{Kinematic Unicycle Robot Model}
\label{sec.KinematicUnicycleRobotModel}

Consider a kinematic unicycle robot moving in a two-dimensional planar Euclidean space $\R^2$ whose state is represented by its position $\pos \in \R^2$ and forward orientation angle $\ort \in [ -\pi, \pi )$, measured in radians counterclockwise from the horizontal axis.
The equations of nonholonomic motion of the kinematic unicycle robot model are given by
\begin{align} \label{eq.UnicycleDynamics}
\dot{\pos} = \linvel \ovect{\ort} \quad \text{and} \quad \dot{\ort} = \angvel 
\end{align}
where  $\linvel \in \R$ and $\angvel \in \R$ are the scalar control inputs that respectively specify the linear and angular velocity of the unicycle robot.
Note that the kinematic unicycle robot model is underactuated (i.e.,  it has three state variables, but only two control inputs) and is subject to the nonholonomic motion constraint of no sideways motion, i.e., $\nvecTsmall{\ort}\dot{\pos} = 0$.

\subsection{Unicycle Control via Angular Feedback Linearization}

A standard angular navigation objective towards a given goal position $\goal \in \R^2$ involves minimizing the angular heading error $\headingerror_{\goal}(\pos, \ort)$ of a unicycle state $(\pos, \ort) \in \R^2 \times [-\pi, \pi)$, which is defined as the counterclockwise angle from the unicycle heading direction to the line passing through the unicycle position $\pos$ and the goal position $\goal$ as
\begin{equation}\label{eq.Angular_Heading_Error}
\headingerror_{\goal}(\pos, \ort) := \arctan\plist{\!\nvecTsmall{\ort}\!\! (\goal\!-\!\pos) \! \bigg / \! \ovecTsmall{\ort}\!\! (\goal\!-\!\pos)\!\!}\!\!\!
\end{equation}
where $\arctan : \mathbb{R} \rightarrow \blist{-\frac{\pi}{2}, \frac{\pi}{2}} $ denotes the inverse tangent function and $\tr{(.)}$ is the transpose operator. 
To resolve indeterminacy, we set $\headingerror_{\goal}(\pos, \ort) = 0$ for $\pos = \goal$.

Under the unicycle dynamics in \refeq{eq.UnicycleDynamics}, the angular heading error $\headingerror_{\goal}(\pos, \ort)$ away from the goal (i.e., $\pos \!\neq\! \goal$) evolves as\reffn{fn.AngularHeadingErrorDynamics}
\begin{align}\label{eq.Angular_Heading_Error_Dynamics}
\dot{\headingerror}_{\goal}(\pos, \ort) =  -\angvel + \linvel \nvecTsmall{\ort}\! \frac{\goal-\pos}{\norm{\goal - \pos}^2}  
\end{align} 
which follows from the chain and quotient rules of differentiation and the standard trigonometric differentiation identities. 
Hence, following a greedy navigation strategy for decreasing the Euclidean distance to the goal \cite{astolfi_JDSMC1999} and angular feedback linearization \cite{lee_etal_IROS2000}, we design a bidirectional unicycle motion controller, denoted by $\ctrl_{\goal}(\pos, \ort) = \plist{\linvel_{\goal}(\pos,\ort), \angvel_{\goal}(\pos,\ort)}$, that determines the linear velocity input $\linvel_{\goal}(\pos,\ort)$ and the angular velocity input $\angvel_{\goal}(\pos,\ort)$  for the kinematic unicycle model in \refeq{eq.UnicycleDynamics} to move towards to the goal position $\goal$ as\reffn{fn.Angular_Heading_Error_Trigonometry}
\begin{subequations}\label{eq.Bidirectional_Unicycle_Control}
\begin{align}
\linvel_{\goal}(\pos, \ort) &= \lingain \ovecTsmall{\ort}\! \!(\goal-\pos)\!\!
\\
\angvel_{\goal}(\pos, \ort) &= \anggain \headingerror_{\goal}(\pos, \ort) + \frac{\lingain}{2}\sin(2\headingerror_{\goal}(\pos, \ort)) 
\\
& \hspace{-9mm}=  \anggain \headingerror_{\goal}(\pos, \ort) + \lingain \nvecTsmall{\ort}\! \frac{\goal-\pos}{\norm{\goal \!-\! \pos}}  \ovecTsmall{\ort}\! \frac{\goal\!-\!\pos}{\norm{\goal\! - \!\pos}}  \!\! \!\nonumber 
\end{align}
\end{subequations}
where  $\headingerror_{\goal}(\pos, \ort)$ is the angular heading error defined in \refeq{eq.Angular_Heading_Error},  and $\lingain > 0$ and $\anggain >0$ are positive scalar control gains for the linear and angular velocity, respectively.
Note that the bidirectional unicycle controller in \refeq{eq.Bidirectional_Unicycle_Control} steers the unicycle either forward or backward, depending on which direction allows the robot to decrease its distance to the goal as
\begin{align}\label{eq.Distance_To_Goal_Dynamics}
\frac{\diff}{\diff t} \norm{\goal - \pos}^2 = - 2\lingain \plist{\ovecTsmall{\ort}\! \!(\goal-\pos)\!}^{\!2} \leq 0. 
\end{align}
Additionally, this bidirectional unicycle controller in \refeq{eq.Bidirectional_Unicycle_Control} ensures linear heading error dynamics as 
\begin{align}\label{eq.Linear_Angular_Heading_Error_Dynamics}
\dot{\headingerror}_{\goal}(\pos, \ort) = - \anggain \headingerror_{\goal}(\pos, \ort).
\end{align}
\addtocounter{footnote}{1}
\footnotetext{\label{fn.AngularHeadingErrorDynamics}The time rate of change of the angular heading error can be obtained as
\begin{align*}
\dot{\headingerror}_{\goal}(\pos, \ort) &= \frac{\diff}{\diff t}  \arctan\plist{\!\nvecTsmall{\ort}\!\! (\goal\!-\!\pos) \! \bigg / \! \ovecTsmall{\ort}\!\! (\goal\!-\!\pos)\!\!} 
\\
& = \frac{1}{\norm{\goal - \pos}^2}  \ovecTsmall{\ort}\!\! (\goal\!-\!\pos) \frac{\diff}{\diff t} \plist{\nvecTsmall{\ort}\!\! (\goal\!-\!\pos)\!\!} 
\\
& \quad \quad  - \frac{1}{\norm{\goal - \pos}^2} \nvecTsmall{\ort}\!\! (\goal\!-\!\pos) \frac{\diff}{\diff t} \plist{\ovecTsmall{\ort}\!\! (\goal\!-\!\pos)\!\!} 
\\
& = -\frac{1}{\norm{\goal - \pos}^2} \plist{\!\!\plist{\ovecTsmall{\ort}\!\! (\goal\!-\!\pos)\!\!}^{\!\!2} \!\! + \! \plist{\nvecTsmall{\ort}\!\! (\goal\!-\!\pos)\!\!}^{\!\!2}} \angvel 
\\
& \quad \quad +  \frac{1}{\norm{\goal - \pos}^2} \nvecTsmall{\ort}\!\! (\goal\!-\!\pos) \linvel
\\
& = - \angvel + \linvel \nvecTsmall{\ort}\! \frac{\goal-\pos}{\norm{\goal - \pos}^2}
\end{align*}
using the following relations
\begin{align*}
\norm{\goal - \pos}^2 &= \plist{\nvecTsmall{\ort}\!\! (\goal-\pos)\!\!}^2 +  \plist{\ovecTsmall{\ort}\!\! (\goal-\pos)\!\!}^2
\\
\frac{\diff}{\diff t} \nvecTsmall{\ort}\!\! (\goal\!-\!\pos) & = -\angvel\ovecTsmall{\ort}\!\! (\goal\!-\!\pos) + \linvel\nvecTsmall{\ort} \ovectsmall{\ort} 
\\
& = -\angvel\ovecTsmall{\ort}\!\! (\goal\!-\!\pos)
\\
\frac{\diff}{\diff t} \ovecTsmall{\ort}\!\! (\goal\!-\!\pos) & = \angvel  \nvecTsmall{\ort}\!\! (\goal\!-\!\pos) - \linvel\ovecTsmall{\ort} \ovectsmall{\ort} 
\\
& =  \angvel  \nvecTsmall{\ort}\!\! (\goal\!-\!\pos) - \linvel.
\end{align*}
}
\addtocounter{footnote}{1}
\footnotetext{\label{fn.Angular_Heading_Error_Trigonometry}It follows from the definition of $\headingerror_{\goal}(\pos, \ort)$ in \refeq{eq.Angular_Heading_Error} that
\begin{align*}
\tfrac{1}{2} \sin\plist{2 \headingerror_{\goal}(\pos, \ort)} &= \sin(\headingerror_{\goal}(\pos, \ort))\cos(\headingerror_{\goal}(\pos, \ort)) 
\\
&= \nvecTsmall{\ort}\! \frac{\goal \! - \! \pos}{\norm{\goal  \! -\! \pos}}  \ovecTsmall{\ort}\! \frac{\goal \! - \! \pos}{\norm{\goal \!-\! \pos}}. 
\end{align*}
}
Therefore, one can conclude the global convergence of the bidirectional unicycle control from the decreasing Euclidean distance to the goal and the angular heading error as follow.

\vspace{-2mm}

\begin{lemma} \label{lem.Global_Convergence}
\emph{(Global Convergence)}
The bidirectional unicycle motion control $\ctrl_{\goal}$ in \refeq{eq.Bidirectional_Unicycle_Control} asymptotically brings all unicycle states \mbox{$(\pos, \ort) \!\in\! \R^{2}\!\!\times\! [-\pi, \pi)$} to any given goal position \mbox{$\goal \!\in\! \R^2$}, i.e., the closed-loop unicycle position trajectory  $\pos(t)$ satisfies
\begin{equation}\label{eq.Global_Convergence}
\lim_{t \rightarrow \infty} \pos(t)  = \goal.
\end{equation}
\end{lemma}
\begin{proof}
See \refapp{app.Global_Convergence}.
\end{proof}
\noindent Most existing unicycle control methods  \cite{astolfi_JDSMC1999, astolfi_SCL1996, lee_etal_IROS2000, deluca_oriolo_vandittelli_IFAC2000, das_etal_TRA2002, novel_campion_bastin_IJRR1995} are capable of establishing a global convergence guarantee to any given goal position. 
However, there are few examples \cite{isleyen_vandewouw_arslan_CDC2023} that allow for the estimation of the final orientation and the total turning effort during the motion. 
In \refsec{sec.Unicycle_Motion_Prediction} below, we demonstrate how angular feedback linearization in \refeq{eq.Linear_Angular_Heading_Error_Dynamics} facilitates the estimation of the total turning and motion range of the closed-loop unicycle motion.

\section{\!\!\!\! Unicycle Total Turning {\small and} Motion Prediction\!\!}
\label{sec.Unicycle_Motion_Prediction}

In this section, we show that the total turning effort of the unicycle control by angular feedback linearization can be explicitly determined in terms of the initial angular heading error and control parameters. This enables a proper selection of control gains to achieve (when desired, e.g., for exploration) or avoid (when undesired, e.g., for minimizing control effort) spiral circulation around the goal and establish an accurate motion range bound on the unicycle motion. 

\subsection{Unicycle Total Turning Effort}
\label{sec.Unicycle_Total_Turning_Effort}

The closed-loop linear heading error dynamics in \refeq{eq.Linear_Angular_Heading_Error_Dynamics} enable the explicit determination of the signed total turning effort of the bidirectional unicycle controller in \refeq{eq.Bidirectional_Unicycle_Control}.

\begin{proposition}\label{prop.Total_Turning_Effort}
\emph{(Total Turning Effort)} Starting at $t = 0$ from any initial unicycle state $(\pos_0, \ort_{0}) \in \R^2 \times [-\pi, \pi)$ towards any goal position $\goal \in \R^2$, the signed total turning effort of the unicycle control $\ctrl_{\goal}$ in \refeq{eq.Bidirectional_Unicycle_Control} along the closed-loop trajectory $(\pos(t), \ort(t))$ is defined as the infinite integral of the angular velocity input $\angvel_{\goal}(\pos(t), \ort(t))$ and is explicitly given by
\begin{subequations} \label{eq.Total_Turning_Effort}
\begin{align}
\totalturning_{\goal}(\pos_0, \ort_0)&:= \int_{0}^{\infty} \angvel_{\goal}(\pos(t), \ort(t)) \diff t 
\\
&\,\,=   \headingerror_{\goal}(\pos_0, \ort_{0}) + \frac{\lingain}{2\anggain}\ \Si(2\headingerror_{\goal}(\pos_0, \ort_{0}))
\end{align}
\end{subequations}
where $\headingerror_{\goal}(\pos, \ort) \! \in\! [-\tfrac{\pi}{2}, \tfrac{\pi}{2}]$ is the angular heading error function  in \refeq{eq.Angular_Heading_Error}, $\lingain$ and $\anggain$ are constant positive control gains, and $\Si(x):=\int_0^x  \frac{\sin(t)}{t} \diff t $ is the sine integral function.
\end{proposition}
\begin{proof}
See \refapp{app.Total_Turning_Effort}.
\end{proof}

Note that the magnitude of the signed total turning effort is the same as the total absolute turning, i.e.,
{
\begin{align*}
\absval{\totalturning_{\goal}(\pos_0, \ort_0)} \!=\! \absval{\int_{0}^{\infty} \!\!\! \!\angvel_{\goal}(\pos(t), \ort(t)\!) \diff t} 
\!=  \!\int_{0}^{\infty} \!\!\!\! \absval{\angvel_{\goal}(\pos(t), \ort(t)\!)} \diff t
\end{align*}
}%
which is due to the monotonicity of the angular velocity input $\angvel_{\goal}(\pos, \ort)$ in \refeq{eq.Bidirectional_Unicycle_Control} with respect to the angular heading error $\headingerror_{\goal}(\pos, \ort)$, and the linearity of the angular heading error dynamics $\dot{\headingerror}_{\goal}(\pos, \ort)$ in \refeq{eq.Linear_Angular_Heading_Error_Dynamics}.
Moreover, the magnitude of the total turning effort $\totalturning_{\goal}(\pos_0, \ort_0)$ can be expressed and linearly bounded from above and below in terms of the magnitude of the initial angular heading error $\headingerror_{\goal}(\pos_0, \ort_0)$ as
\begin{align}
\!\absval{\totalturning_{\goal}(\pos_0, \ort_0)} &= \absval{\headingerror_{\goal}(\pos_0, \ort_0)} + \tfrac{\lingain}{2\anggain} \absval{\Si(2\headingerror_{\goal}(\pos_0, \ort_0))} \!\!\! \label{eq.Turning_Effort_Heading_Error_Equality}
\\
\!\absval{\headingerror_{\goal}(\pos_0, \ort_0)} &\leq \absval{\totalturning_{\goal}(\pos_0, \ort_0)} \leq (1 \! +\! \tfrac{\lingain}{\anggain}) \absval{\headingerror_{\goal}(\pos_0, \ort_0)} \!\!\! \label{eq.Turning_Effort_Heading_Error_Inequality}
\end{align}
which follows from the fact that the sine integral function $\Si(x)$ is monotone increasing over $\blist{-\pi, \pi}$, since $\frac{\sin(x)}{x} \geq 0$ for any $x \in [-\pi, \pi]$, and it is linearly bounded as $\absval{\Si(x)} \leq \absval{x}$ over $[-\pi, \pi]$, as illustrated in \reffig{fig.Sine_Integral}.
Therefore, in cases where it is undesirable, for example, to minimize control effort and travel distance, one can avoid spiral circulation around the goal by setting $\lingain \leq \anggain$, ensuring that $\absval{\totalturning_{\goal}(\pos_0, \ort_0)} \! \leq \! 2\absval{\headingerror_{\goal}(\pos_0, \ort_0)}\! \leq\! \pi$. Similarly, in cases where it is desirable to explore the goal region while approaching the goal, as seen in nature with insects \cite{boyadzhiev_CMJ1999, muller_wehner_JCP1994}, one can achieve spiral circulation around the goal by setting $\lingain > \anggain$.

\begin{figure}[t]
\centering
\includegraphics[width = 0.85\columnwidth]{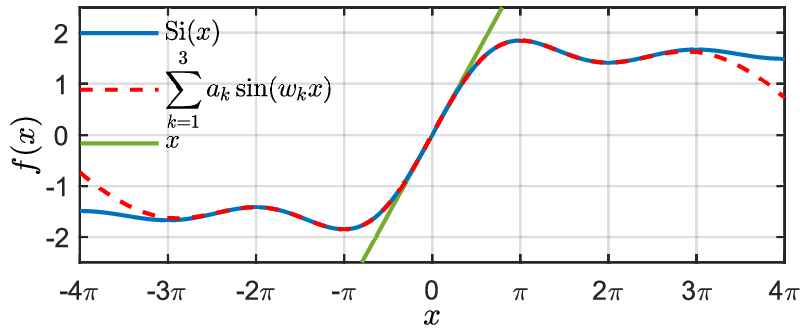} 
\vspace{-3.5mm}
\caption{Approximation of the sine integral function $\si(x)$ over $[-\pi, \pi]$ using a weight sum of three sinusoidal functions, where the optimal weights and frequencies are obtained using nonlinear least squares optimization as   $a_1 \!=\! 1.964 , a_2 \!=\! 0.553 , a_3 \! =\! 0.189$ and  $\omega_1 \!=\! 0.235, \omega_2 \!=\! 0.656, \omega_3 \!=\! 0.931$, corresponding to a root mean square error of $6.1 \times 10^{-4}$.}
\label{fig.Sine_Integral}
\vspace{-4mm}
\end{figure}

The explicit form of the total turning effort $\totalturning_{\goal}(\pos_0, \ort_0)$ in \refeq{eq.Total_Turning_Effort} also allows for determining the final unicycle orientation, denoted by $\ort^*_{\goal}(\pos_0, \ort_0)$, when the robot asymptotically reaches to the goal $\goal$ (up to the equivalence of angles) as\footnote{A unicycle orientation $\ort$ is an element of $\R/ \sim$ where two angles $\alpha, \beta \in \R$ are equivalent, denoted by $\alpha \sim \beta$, if and only if $\alpha = \beta + 2\pi k$ for some integer $k \in \Z$. Hence, the final orientation $\ort^*_{\goal}(\pos_0, \ort_0)$ in $[-\pi, \pi)$ satisfies
\begin{align*} \label{eq.Final_Orientation_Angle}
\fort_{\goal}(\pos_0, \ort_0) & \sim \lim_{t \rightarrow \infty} \ort(t) = \ort_0 + \totalturning_{\goal}(\pos_0, \ort_0)
\\ 
& = \!\! \mod\plist{\ort_0 + \totalturning_{\goal}(\pos_0, \ort_0) + \pi, 2\pi} - \pi
\end{align*}
where $\mod$ denotes the modulo operator.
}
\begin{align}
\ort^*_{\goal}(\pos_0, \ort_0) := \lim_{t \rightarrow \infty} \ort(t) &= \ort_0 + \totalturning_{\goal}(\pos_0, \ort_0).
\end{align}
The final unicycle orientation plays a key role in accurately bounding the closed-loop unicycle position trajectory later in \refprop{prop.Triangular_Motion_Bound} since it determines the approach angle to the goal for $\frac{\lingain}{2\anggain} \Si(2 \headingerror_{\goal}(\pos_0, \ort_0)\!) \!\in\! [-\frac{\pi}{2}, \frac{\pi}{2}]$ (e.g., when $\lingain \leq \anggain$) because  the angular heading error of the final orientation  is%
\footnote{The angular heading error of the final unicycle orientation can be over-approximated  using the linear bound $\absval{\Si(x)} \leq \absval{x}$ for $x \in [-\tfrac{\pi}{2}, \tfrac{\pi}{2}]$~as 
\begin{align*}
|\headingerror_{\goal}(\pos_0, \ort^*_{\goal}(\pos_0, \ort_0)) | \leq \tfrac{\lingain}{\anggain} \absval{\headingerror_{\goal}(\pos_0, \ort_0)} 
\end{align*}
for $\lingain \leq \anggain$, which can be used to overestimate the unicycle heading line intersections in \refeq{eq.Heading_Line_Intersection} to construct an analytical over-approximation of the triangular motion bounds in \refeq{eq.Triangular_Motion_Bound} and \refeq{eq.Diamond-Shaped_Motion_Prediction} without using the $\Si$ function.
} 
\begin{align}
\headingerror_{\goal}(\pos_0, \ort^*_{\goal}(\pos_0, \ort_0)) = - \tfrac{\lingain}{2\anggain} \Si(2 \headingerror_{\goal}(\pos_0, \ort_0))
\end{align}
for  $\frac{\lingain}{2\anggain} \Si(2 \headingerror_{\goal}(\pos_0, \ort_0)) \in [-\frac{\pi}{2}, \frac{\pi}{2}]$,  which is due to the following properties of the angular heading error $\headingerror_{\goal}(\pos, \ort)$
\begin{align}
&\headingerror_{\goal}(\pos, \ort + \headingerror_{\goal}(\pos, \ort)) = 0
\\
&\headingerror_{\goal}(\pos, \ort + \headingerror_{\goal}(\pos, \ort) + \ort') = - \ort' \quad \forall \ort' \in [-\tfrac{\pi}{2}, \tfrac{\pi}{2}]
\end{align} 
Hence, for $\lingain \leq \anggain$, the total turning effort and the current and final angular heading errors are related to each other as 
\begin{subequations} \label{eq.Total_Turning_Initila_Final_Heading_Error_Equality} 
\begin{align}
\totalturning_{\goal}(\pos, \ort) &= \headingerror_{\goal}(\pos, \ort) - \headingerror_{\goal}(\pos, \fort_{\goal}(\pos, \ort)\!)
\\
\absval{\totalturning_{\goal}(\pos, \ort)} &= \absval{\headingerror_{\goal}(\pos, \ort)} +\absval{\headingerror_{\goal}(\pos, \fort_{\goal}(\pos, \ort)\!)}.
\end{align}
\end{subequations}
since the current and final heading errors, respectively, have  the same and opposite signs with the total turning effort, i.e.,  $\totalturning_{\goal}(\pos, \ort) \headingerror_{\goal}(\pos, \ort)\! \geq\! 0$ and $\headingerror_{\goal}(\pos, \ort) \headingerror_{\goal}(\pos, \fort_{\goal}(\pos, \ort)\!) \!\leq\! 0$.

\begin{table}[t]
\caption{Sine Integral Approximation as a Sum of Sinusoidals \\ $\si(x) \approx \sum_{k = 1}^{n} a_k \sin (\omega_k x)$}
\label{tab.Sine_Integral_Approximation}
  \centering
  \vspace{-3mm}
  \begin{tabular}{|c|c|c|c|c|c|c|c|}
    \hline
    $n$ &$a_1$&$\omega_1$&$a_2$&$\omega_2$&$a_3$&$\omega_3$  & RMSE \\
    \hline
    1 &1.839 &0.535 & -&- &- &- & $ 8.0\! \times\! 10^{-3}\!$\\
    \hline
    2 &1.931 &0.330 &0.424 &0.854 & - &-& $ 1.3 \!\times\! 10^{-3}\!$\\
    \hline
    3 &1.964 &0.235 &0.553 &0.656 &0.189 &0.931  &$ 6.1 \!\times\! 10^{-4}\!$\\
    \hline
  \end{tabular}
  \vspace{-3mm}
\end{table}


Finally, it is useful to note that the sine integral function is bounded above by $2$ and satisfies $\Si(\infty) = \frac{\pi}{2}$ \cite{nikiforov1988special}. 
However, the sine integral $ \Si(x) = \int_0^x  \frac{\sin(t)}{t} \diff t $ is a well-known example of an integral that does not have an elementary function as its anti-derivative \cite{nikiforov1988special}.
It is available as a built-in function in many programming languages (e.g., $\mathtt{sinint}$ in MATLAB).
One can also accurately approximate the sine integral function over $[-\pi, \pi]$ as a weighted combination of sinusoidals by applying nonlinear least squares optimization to minimize $\int\limits_{-\pi}^{\pi} \!\!\plist{\si(x) \!-\! \sum\nolimits_{k=0}^{n} \!a_k\sin(\omega_k x)\!}^2  dx$ as in \reftab{tab.Sine_Integral_Approximation} and \reffig{fig.Sine_Integral}.

\subsection{Unicycle Motion Range Prediction}
\label{sec.Unicycle_Motion_Range_Prediction}

The monotone decrease of the distance to the goal in \refeq{eq.Distance_To_Goal_Dynamics}, the exponential decay of the angular heading error in \refeq{eq.Angular_Heading_Error_Dynamics}, the explicit form of the total turning error in \refeq{eq.Total_Turning_Effort} and the final orientation angle in \refeq{eq.Final_Orientation_Angle} allow us to establish circular, conic, and triangular motion range bounds on the closed-loop unicycle position trajectory. 
To highlight  key geometric characteristics of the closed-loop unicycle motion, we find it useful to introduce several fundamental geometric elements that define the unicycle motion range, illustrated in \reffig{fig.Unicycle_Motion_Prediction_Elements}.  


\begin{lemma}\label{lem.Projected_Goal}
\emph{(Projected Goal on Heading Line)}
The closest point $\cpos_{\goal}(\pos, \ort)$ of the heading line $\headingline(\pos, \ort)$ of a unicycle state $(\ort, \pos)$ to the goal position $\goal$, and its reflection $\cposr_{\goal}(\pos, \ort)$ with respect to the goal line $\blist{\pos, \goal}$ are given in terms of the angular heading error in \refeq{eq.Angular_Heading_Error} by
{\small
\begin{subequations}\label{eq.Projected_Goal}
\begin{align}
&\!\!\scalebox{0.97}{$
\cpos_{\goal}(\pos, \ort)\!:= \pos + \cos(\headingerror_{\goal}(\pos,\ort)\!) \Rmat(-\headingerror_{\goal}(\pos,\ort)\!) (\goal \!-\! \pos)$} \!\!  
\\
&
\!\!\scalebox{0.97}{$\cposr_{\goal}(\pos, \ort)\!:= \pos + \cos(\headingerror_{\goal}(\pos,\ort)\!) \Rmat(+\headingerror_{\goal}(\pos,\ort)\!) (\goal \!-\! \pos)$} \!\!
\end{align}
\end{subequations}
}%
where $\headingline(\pos, \ort)\!:=\!\clist{\pos \!+\! \alpha \scalebox{0.85}{$\ovectsmall{\ort}$} \Big | \alpha \!\in\! \R }$ is the line that passes through $\pos$ with orientation $\ort$, the goal line passing through $\pos$ and $\goal$ is denoted by $\blist{\pos, \goal}\!:=\!\clist{\alpha \pos \!+\! (1 \!-\! \alpha) \goal \big| \alpha \!\in\! \R}$, and  \mbox{$\Rmat(\theta)\! := \!\scalebox{0.8}{$\begin{bmatrix}
\cos \theta & - \sin \theta \\ \sin \theta & \cos\theta
\end{bmatrix}$}$} is the 2D rotation matrix. 
\end{lemma}
\begin{proof}
See \refapp{app.Projected_Goal}.
\end{proof}


\begin{lemma}\label{lem.Heading_Line_Intersection}
\emph{(Heading Line Intersection)}
For control gains $\lingain \leq \anggain$,  the intersection point $\xpos_{\goal}(\pos, \ort)$ of the current heading line $\headingline(\pos, \ort)$  and the final heading line $\headingline(\goal, \fort_{\goal}(\pos, \ort))$  of a unicycle state $(\pos, \ort)$, and its reflection $\xposr_{\goal}(\pos, \ort)$ with respect to the goal line $\blist{\pos, \goal}$ are given by\reffn{fn.HeadingLineIntersection}
{\small
\begin{subequations} \label{eq.Heading_Line_Intersection}
\begin{align}
\!\!\scalebox{0.95}{$\xpos_{\goal}(\pos, \ort)$}  
& \scalebox{0.95}{$:= \pos \!-\! \scalebox{0.9}{$\dfrac{\sin(\headingerror_{\goal}(\pos, \fort_{\goal}(\pos, \ort)\!)\!)}{\sin(\totalturning_{\goal}(\pos, \ort)\!)}$}  \Rmat(-\headingerror_{\goal}(\pos, \ort)\!) (\goal\! - \!\pos)$}\! \nonumber
\\
& \scalebox{0.95}{$ = \! \goal \!-\! \scalebox{0.85}{$\dfrac{\sin(\headingerror_{\goal}(\pos, \ort)\!)}{\sin(\totalturning_{\goal}(\pos, \ort)\!)}$} \Rmat(-\headingerror_{\goal}(\pos, \fort_{\goal}(\pos, \ort)\!)\!) (\goal \!-\! \pos)$} \!\!
\\
\!\!\scalebox{0.95}{$\xposr_{\goal}(\pos, \ort)$} 
& \scalebox{0.95}{$:= \pos \!-\! \scalebox{0.9}{$\dfrac{\sin(\headingerror_{\goal}(\pos, \fort_{\goal}(\pos, \ort)\!)\!)}{\sin(\totalturning_{\goal}(\pos, \ort)\!)}$}  \Rmat(+\headingerror_{\goal}(\pos, \ort)\!) (\goal\! - \!\pos)$}\! \nonumber
\\
& \scalebox{0.95}{$ = \goal \!-\! \scalebox{0.85}{$\dfrac{\sin(\headingerror_{\goal}(\pos, \ort))}{\sin(\totalturning_{\goal}(\pos, \ort))}$} \Rmat(+\headingerror_{\goal}(\pos, \fort_{\goal}(\pos, \ort)\!)\!) (\goal \!-\! \pos)$}.  \!
\end{align}
\end{subequations}
}%
\end{lemma}
\begin{proof}
See \refapp{app.Heading_Line_Intersection}.
\end{proof}

\begin{figure}[t]
\centering
\begin{tabular}{@{}c@{\hspace{2mm}}c@{}}
\multirow{2}{*}[ 0.215\columnwidth]{\includegraphics[width = 0.515\columnwidth]{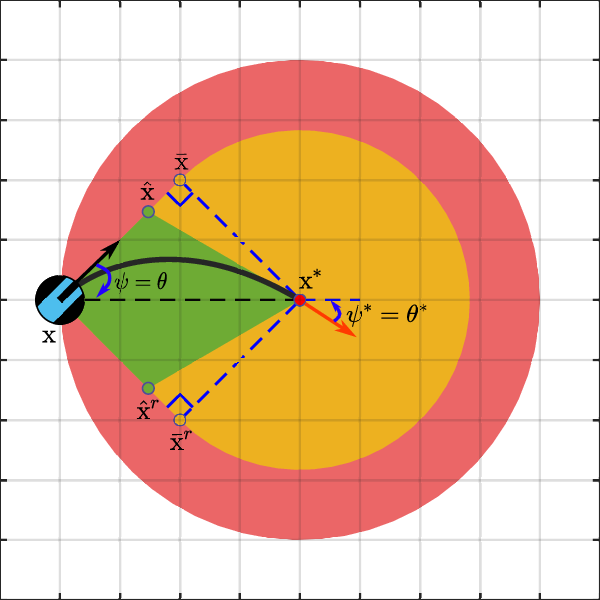}} & 
\includegraphics[width = 0.25\columnwidth]{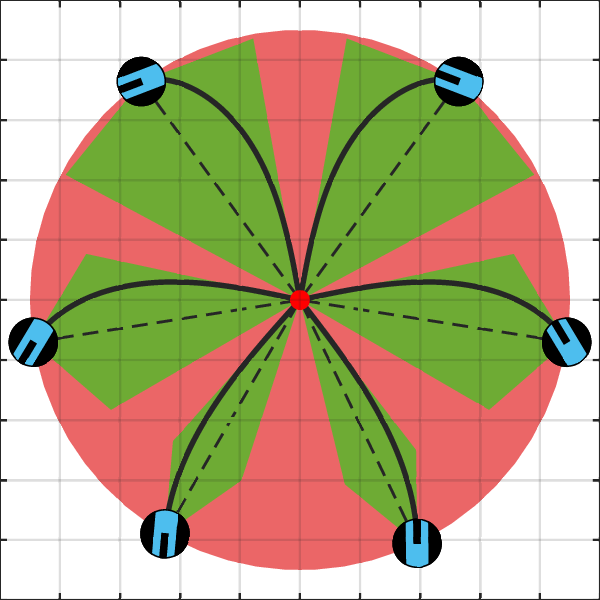} 
\\ 
& \includegraphics[width = 0.25\columnwidth]{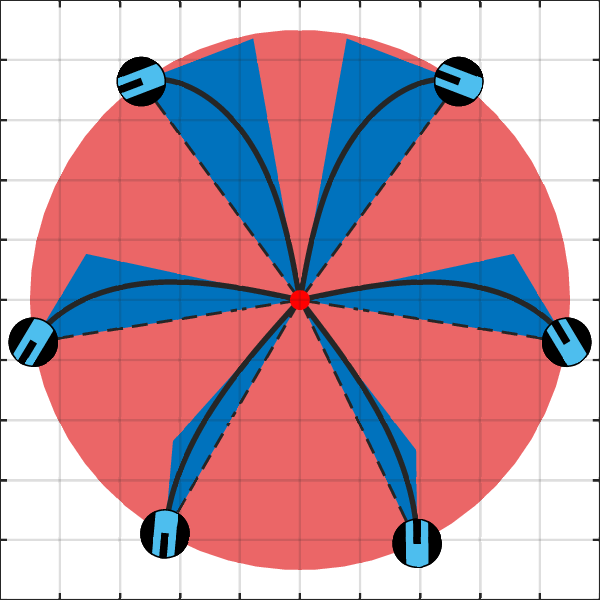} 
\end{tabular}
\vspace{-1mm}
\caption{Key characteristic geometric points of the closed-loop unicycle motion determining triangular motion bounds (blue), which are symmetrized to obtain diamond-shaped motion prediction (green), are compared to the circular (red) and conic (orange) motion predictions. The initial and final orientation and heading errors, relative to the straight line between the initial position and the goal positions, are depicted with the black and red arrows.}
\label{fig.Unicycle_Motion_Prediction_Elements}
\vspace{-3mm}
\end{figure}

A bounded total turning effort, for example, $\absval{\totalturning_{\goal}(\pos, \ort)} < \pi$, can be achieved by setting control gains as  $\lingain \leq \anggain$.
This facilitates the establishment of a simple triangular motion range bound on the closed-loop unicycle position trajectory.

\begin{proposition} \label{prop.Triangular_Motion_Bound}
(Triangular Motion Bound)
If the linear velocity gain is less than or equal to the angular velocity gain, i.e., $\lingain \leq \anggain$, the closed-loop unicycle position trajectory $\pos(t)$ under the unicycle control $\ctrl_{\goal}$ in \refeq{eq.Bidirectional_Unicycle_Control}, starting at $t = 0$ from any unicycle state $(\pos_0, \ort_0) \in \R^2 \times [-\pi, \pi)$ towards any given goal position $\goal \in \R^2$, is contained in the convex hull of the initial unicycle position $\pos_0$, the goal position $\goal$, and the intersection point $\xpos_{\goal}(\pos_0, \ort_0)$ of the initial and final  heading lines in \refeq{eq.Heading_Line_Intersection}, i.e., 
\begin{align}\label{eq.Triangular_Motion_Bound}
\pos(t) \in \conv\plist{\pos_0, \goal, \xpos_{\goal}(\pos_0, \ort_0)} \quad \forall t \geq 0
\end{align}
where $\conv$ denotes the convex hull operator.
\end{proposition}
\begin{proof}
See \refapp{app.Triangular_Motion_Bound}.
\end{proof}

\addtocounter{footnote}{1}
\footnotetext{\label{fn.HeadingLineIntersection}
For any unicycle state $(\pos, \ort)$ with nonzero angular heading error relative to the goal $\goal$, i.e., $\headingerror_{\goal}(\pos, \ort) \neq 0$, the intersection point of the current and final heading lines, $\headingline(\pos, \ort)$ and $\headingline(\goal, \fort)$, where $\fort := \fort_{\goal}(\pos, \ort)$,  can be alternatively determined as
\begin{align*}
\xpos_{\goal}(\pos,\ort)& = \pos \!-\! \scalebox{0.88}{$\dfrac{\nvecTsmall{\ort^*}\!\! (\goal\! -\! \pos)}{\nvecTsmall{\ort^*}\!\ovectsmall{\ort}}$}\ovectsmall{\ort} = \goal \!-\! \scalebox{0.88}{$\dfrac{\nvecTsmall{\ort}\!\! (\goal\! -\! \pos)}{\nvecTsmall{\ort}\!\ovectsmall{\ort^*}}$}\ovectsmall{\ort^*}.
\end{align*}
}

Note that  the triangular motion range bound change discontinuously when the angular heading error magnitude is $\pi/2$ (i.e., $\absval{\headingerror_{\goal}(\pos, \ort)} = \frac{\pi}{2}$), with respect to changes in unicycle state and goal position, as seen in \reffig{fig.Unicycle_Motion_Prediction_Elements}.
Hence, below, we consider a symmetrized version of the triangular motion bound by reflecting it around the line passing through the unicycle position and the goal location. 
This results in a continuous feedback motion prediction that ensures Lipschitz continuity in its distance to any given point (see  \refprop{prop.Motion_Prediction_Distance}), which is critical for safe and smooth motion control \cite{isleyen_vandewouw_arslan_RAL2022}.

\begin{definition}\label{def.Diamond-Shape_Motion_Prediction}
\emph{(Diamond-Shaped Motion Prediction)} 
To continuously bound the closed-loop unicycle motion under the unicycle control $\ctrl_{\goal}$ in \refeq{eq.Bidirectional_Unicycle_Control}, for any $\lingain \leq \anggain$ and unicycle state $(\pos, \ort)$, we define the diamond-shaped unicycle motion prediction, denoted by $\motionset_{\ctrl_{\goal, \dshape}}(\pos, \ort)$, as the convex hull of the unicycle position $\pos$, the goal position $\goal$, the intersection point $\xpos_{\goal}(\pos, \ort)$ of the current and final heading lines, and its reflection $\xposr_{\goal}(\pos, \ort)$  with respect to the goal line in \refeq{eq.Heading_Line_Intersection} as
\begin{align}\label{eq.Diamond-Shaped_Motion_Prediction}
\motionset_{\ctrl_{\goal, \dshape}}(\pos, \ort) := \conv\plist{\pos, \goal, \xpos_{\goal}(\pos, \ort), \xposr_{\goal}(\pos, \ort)}.
\end{align}
%
\end{definition}


As the unicycle control in \refeq{eq.Bidirectional_Unicycle_Control} continuously reduces the positional distance to the goal and consistently aligns the orientation with the goal by decreasing the angular heading error, the closed-loop unicycle position trajectory can also be bounded by circular and conic motion sets \cite{isleyen_vandewouw_arslan_IROS2023}. 

\begin{proposition} \label{prop.Circular_Conic_Motion_Prediction}
\emph{(Circular \& Conic Motion Predictions)}
For $\lingain > 0$ and $\anggain > 0$, starting at $t = 0$ from any unicycle state $(\pos_0, \ort_0) \in \R^{2} \times [-\pi, \pi)$ towards any goal position $\goal \in \R^{2}$, the unicycle position along the closed-loop  unicycle trajectory $(\pos(t), \ort(t))$ of the unicycle dynamics in \refeq{eq.UnicycleDynamics} under the bidirectional unicycle control in \refeq{eq.Bidirectional_Unicycle_Control} is contained for all future times in the conic motion prediction set $\motionset_{\ctrl_{\goal}, \cone}(\pos_0, \ort_0)$, contained in the circular motion prediction set $\motionset_{\ctrl_{\goal}, \ball}(\pos_0, \ort_0)$ (i.e., $\pos(t) \in \motionset_{\ctrl_{\goal}, \cone}(\pos_0, \ort_0) \subseteq \motionset_{\ctrl_{\goal}, \ball}(\pos_0, \ort_0)$ for all $t \geq 0$)  that are, respectively, defined as 
\begin{align}
\motionset_{\ctrl_{\goal}, \cone}(\pos_0, \ort_0) &: = \cone(\pos_0, \goal, \absval{\headingerror_{\goal}(\pos_0, \ort_0)}) \label{eq.Circular_Motion_Prediction}
\\
\motionset_{\ctrl_{\goal}, \ball}(\pos_0, \ort_0) & := \ball(\goal, \norm{\pos_0 - \goal})  \label{eq.Conic_Motion_Prediction}
\end{align}
where $\ball(\ctr, \radius):=\clist{\vect{y} \in \R^{2} | \norm{\vect{y} - \ctr} \leq \radius}$ denotes the closed Euclidean ball centered at $\ctr \in \R^{2}$ with radius $\radius \geq 0$, and \mbox{$\cone(\vect{a}, \vect{b}, \theta) \!:=\! \clist{\vect{a} \!+\! \alpha (\vect{z} \!-\! \vect{a}) \big | \alpha \!\in\! [0,1], \vect{z} \in \ball(\vect{b}, \sin(\theta)\norm{\vect{a} \!-\! \vect{b}}) }$} denotes the closed convex cone\reffn{fn.Convex_Cone} with apex point $\vect{a} \in \R^2$, base point $\vect{b} \in \R^{2}$ and cone angle\reffn{fn.Cone_Angle} $\theta \in [0, \pi/2]$. 
\end{proposition}
\begin{proof}
See \refapp{app.Circular_Conic_Motion_Prediction}.
\end{proof}

Note that the conic motion prediction can be decomposed as a union of a triangle and a circle, using the projected goal point $\cpos_{\goal}(\pos, \ort)$ and its reflection $\cposr_{\goal}(\pos, \ort)$ in \refeq{eq.Projected_Goal}, as 
\begin{align}\label{eq.Conic_Motion_Prediction_Decomposition}
\!\!\motionset_{\goal, \cone}(\pos_0, \ort_0) =  &\, \conv(\pos_0, \cpos_{\goal}(\pos_0, \ort_0), \cposr_{\goal}(\pos_0, \ort_0) ) \nonumber \\
& \quad \quad \quad \cup \ball(\goal\!, \norm{\cpos_{\goal}(\pos_0, \ort_0) - \goal})\!\!\!
\end{align}
which is useful for fast collision checking and distance-to-collision computation.

As a ground truth, it is also convenient to have the exact forward motion set of the closed-loop unicycle motion.

\begin{definition}\label{def.Forward_Simulated_Motion_Prediction}
\emph{(Forward-Reachable Motion Set)} 
To capture the exact future unicycle motion under the unicycle control $\ctrl_{\goal}(\pos, \ort)$ in \refeq{eq.Bidirectional_Unicycle_Control}, starting at $t\! = \!0$ from any unicycle state $(\pos_0, \ort_0)$, we define the unicycle forward-reachable motion set, denoted by $\motionset_{\ctrl_{\goal}, \fwdsim}(\pos_{0}, \ort_{0})$, as
\begin{align}
\!\!\motionset_{\ctrl_{\goal}, \fwdsim}(\pos_{0}, \ort_{0})\!:= \! \left \{   \pos(t) \bigg | \right. &\dot{\pos}(t) = \linvel_{\goal}(\pos(t),\ort(t)\!) \ovectsmall{\ort(t)}, \nonumber
\\
&\dot{\ort}(t) = \angvel_{\goal}(\pos(t), \ort(t)\!), \nonumber
\\
& \pos(0) = \pos_0, \ort(0) = \ort_0, t \geq 0 \left. \bigg. \right\} \!\!\!
\end{align}
where $\linvel_{\goal}(\pos, \ort)$ and $\angvel_{\goal}(\pos, \ort)$ are the unicycle linear and angular velocity control in \refeq{eq.Bidirectional_Unicycle_Control}.
\end{definition}

\addtocounter{footnote}{1}
\footnotetext{\label{fn.Convex_Cone}The convex cone $\cone(\vect{a}, \vect{b}, \theta)$ is the convex hull of the point $\vect{a}$ and the ball $\ball(\vect{b}, \sin(\theta)\norm{\vect{a} - \vect{b}})\!)$, i.e., $\cone(\vect{a}, \vect{b}, \theta) = \conv(\vect{a}, \ball(\vect{b}, \sin(\theta)\norm{\vect{a} - \vect{b}})\!)$.} 

\addtocounter{footnote}{1}
\footnotetext{\label{fn.Cone_Angle}The cone angle is the angle between the line passing through the apex and base points and the cone boundary.}

Note that the forward-reachable motion set of the unicycle controller does not accept a closed-form solution and needs to be numerically computed.
It has some nice positive inclusion and asymptotic radial decay properties, but its minimum distance to any given point might discontinuously change  when the goal position is changed.
We below highlight some useful properties of the proposed unicycle feedback motion prediction methods that are essential for provably correct safe robot motion control \cite{isleyen_vandewouw_arslan_RAL2022, arslan_arXiv2022}.

\begin{proposition}\label{prop.Positive_Inclusion}
\emph{(Positive Inclusion of Motion Prediction)} The circular, conic, diamond-shaped, and forward-reachable motion prediction sets for the unicycle control $\ctrl_{\goal}$ in \refeq{eq.Bidirectional_Unicycle_Control} are all positively inclusive along the closed-loop unicycle state trajectory $(\pos(t), \ort(t))$, as illustrated in \reffig{fig.Positive_Inclusion},  i.e., for any $\motionset_{\ctrl_{\goal}} \in \clist{\motionset_{\ctrl_{\goal}, \ball}, \motionset_{\ctrl_{\goal}, \cone}, \motionset_{\ctrl_{\goal}, \dshape}, \motionset_{\ctrl_{\goal}, \fwdsim}}$
\begin{align}
\motionset_{\ctrl_{\goal}}(\pos(t), \ort(t)) \supseteq \motionset_{\ctrl_{\goal}}(\pos(t'), \ort(t')) \quad \forall t'\geq t. 
\end{align}
\end{proposition}
\begin{proof}
See \refapp{app.Positive_Inclusion}.
\end{proof}

\begin{figure}[t]
\centering
\begin{tabular}{@{}c@{\hspace{0.5mm}}c@{\hspace{0.5mm}}c@{}}
\includegraphics[width = 0.33\columnwidth]{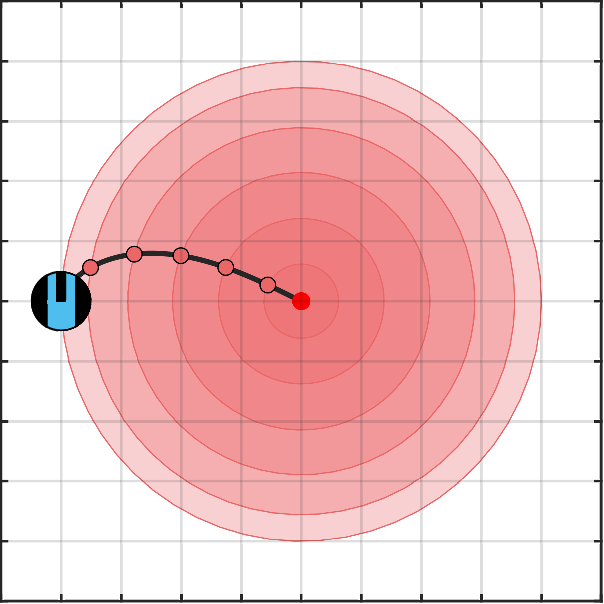} & 
\includegraphics[width = 0.33\columnwidth]{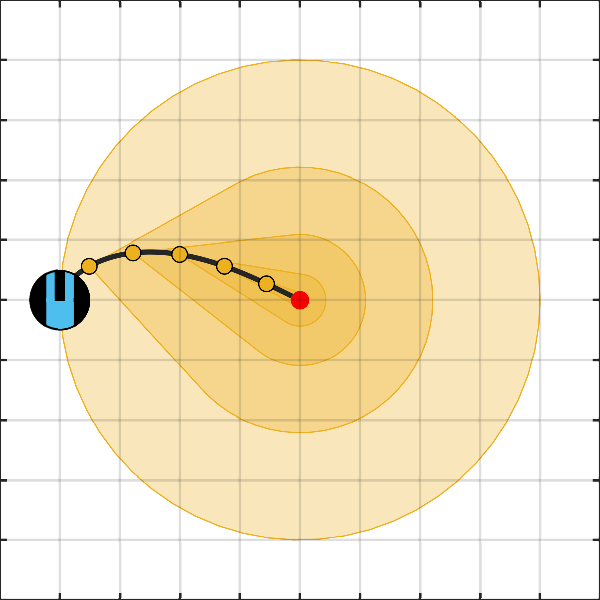} &
\includegraphics[width = 0.33\columnwidth]{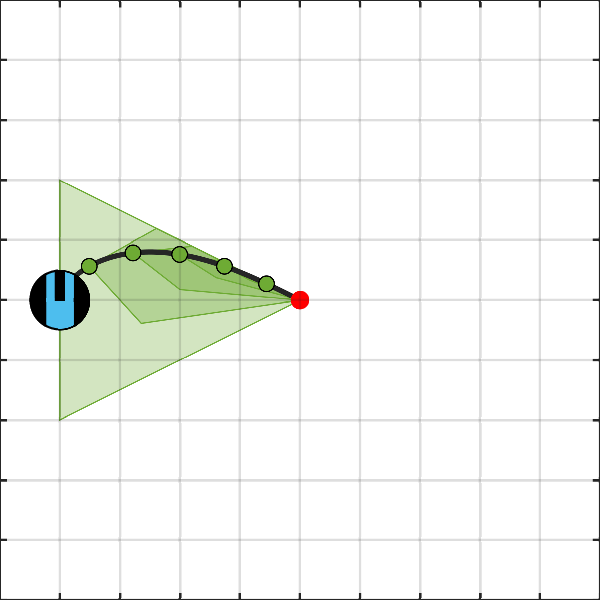}
\end{tabular}
\vspace{-2mm}
\caption{Positive inclusion and radial decay of unicycle feedback motion predictions represented as (left) ball-shaped, (middle) cone-shaped, and (right) diamond-shaped motion bounds on the unicycle trajectory (black~line).}
\label{fig.Positive_Inclusion}
\vspace{-3mm}
\end{figure}

\begin{proposition}\label{prop.Radial_Decay}
\emph{(Radial Decay of Motion Prediction)} Along the closed-loop unicycle state trajectory $(\pos(t), \ort(t))$ of the unicycle controller $\ctrl_{\goal}$ in \refeq{eq.Bidirectional_Unicycle_Control}, the circular, conic, diamond-shaped, and forward-reachable motion prediction sets asymptotically shrink to the goal position $\goal$ as their radii relative to the goal asymptotically decay to zero (see \reffig{fig.Positive_Inclusion}), i.e., for any $\motionset_{\ctrl_{\goal}} \in \clist{\motionset_{\ctrl_{\goal}, \ball}, \motionset_{\ctrl_{\goal}, \cone}, \motionset_{\ctrl_{\goal}, \dshape}, \motionset_{\ctrl_{\goal}, \fwdsim}}$
\begin{align}
\lim_{t \rightarrow \infty} \max_{\pos' \in \motionset_{\ctrl_{\goal}} (\pos(t), \ort(t))} \norm{\pos' - \goal} = 0 
\end{align}  
\end{proposition}
\begin{proof}
See \refapp{app.Radial_Decay}.
\end{proof}

\begin{proposition}\label{prop.Motion_Prediction_Distance}
\emph{(Distance to Motion Prediction)} For any unicycle state $(\pos, \ort) \in \R^2 \times [-\pi, \pi)$, goal position $\goal \in \R^2$ and motion prediction set $\motionset_{\ctrl_{\goal}} \!\in\! \clist{\motionset_{\ctrl_{\goal}, \ball}, \motionset_{\ctrl_{\goal}, \cone}, \motionset_{\ctrl_{\goal}, \dshape}}$, the minimum distance $\min_{\pos'\in \motionset_{\ctrl_{\goal}}(\pos, \ort)} \norm{\vect{y} - \vect{x'}}$ of any point $\vect{y} \! \in\! \R^2$ to the motion prediction set $\motionset_{\goal}(\pos, \ort)$ is a locally Lipschitz continuous\footnote{Here, local Lipschitz continuity is necessary to manage arbitrary continuous changes in the goal position \cite{isleyen_vandewouw_arslan_RAL2022}, as for safe path-following control in \refsec{sec.Safe Unicycle_Path_Following}. This requirement can be relaxed if a discrete-time goal update with an advance safety check is employed.} function of the unicycle position $\pos$, the unicycle orientation $\ort$, the goal position $\goal$, and the point $\vect{y}$.
\end{proposition}
\begin{proof}
See \refapp{app.Motion_Prediction_Distance}.
\end{proof}

Finally, it is useful to highlight the inclusion relation of unicycle feedback motion prediction methods seen in \reffig{fig.Unicycle_Feedback_Motion_Prediction}.

\begin{proposition}\label{prop.Inclusion_Order}
\emph{(Inclusion Order of Motion Predictions)} For control gains $\lingain \leq \anggain$ and any unicycle state $(\pos, \ort)$ with a total turning effort of $\absval{\totalturning_{\goal}(\pos, \ort)} \leq \tfrac{\pi}{2}$ towards the goal position $\goal$, the proposed unicycle feedback motion prediction methods for the unicycle control $\ctrl_{\goal}$ in \refeq{eq.Bidirectional_Unicycle_Control} satisfy
{\small
\begin{align*}
\motionset_{\ctrl_{\goal}, \fwdsim}(\pos, \ort) \subseteq \motionset_{\ctrl_{\goal}, \dshape}(\pos, \ort) \subseteq \motionset_{\ctrl_{\goal}, \cone}(\pos, \ort) \subseteq \motionset_{\ctrl_{\goal}, \ball}(\pos, \ort). 
\end{align*}
}%
\end{proposition}
\begin{proof}
See \refapp{app.Inclusion_Order}.
\end{proof}

\section{Safe Unicycle Path-Following Control}
\label{sec.Safe Unicycle_Path_Following}

In this section, we demonstrate an example application of the unicycle motion controller in \refeq{eq.Bidirectional_Unicycle_Control} and the associated unicycle feedback motion prediction methods for safe path following of a reference path around obstacles using a time governor \cite{arslan_arXiv2022}.
In short, a time governor performs an online continuous time parametrization of a reference path for provably correct and safe path following based on the safety of the predicted robot motion \cite{arslan_arXiv2022}.
The time-governor framework requires a feedback motion prediction method that has asymptotic radial decay (\refprop{prop.Radial_Decay}) and Lipschitz-continuous point distance (\refprop{prop.Motion_Prediction_Distance}) properties, and enjoys positively inclusive motion prediction (\refprop{prop.Positive_Inclusion}). 

\begin{figure*}[ht]
\centering
\begin{tabular}{@{}c@{\hspace{0.005\columnwidth}}c@{\hspace{0.005\columnwidth}}c@{\hspace{0.005\columnwidth}}c }
\includegraphics[width = 0.48\columnwidth]{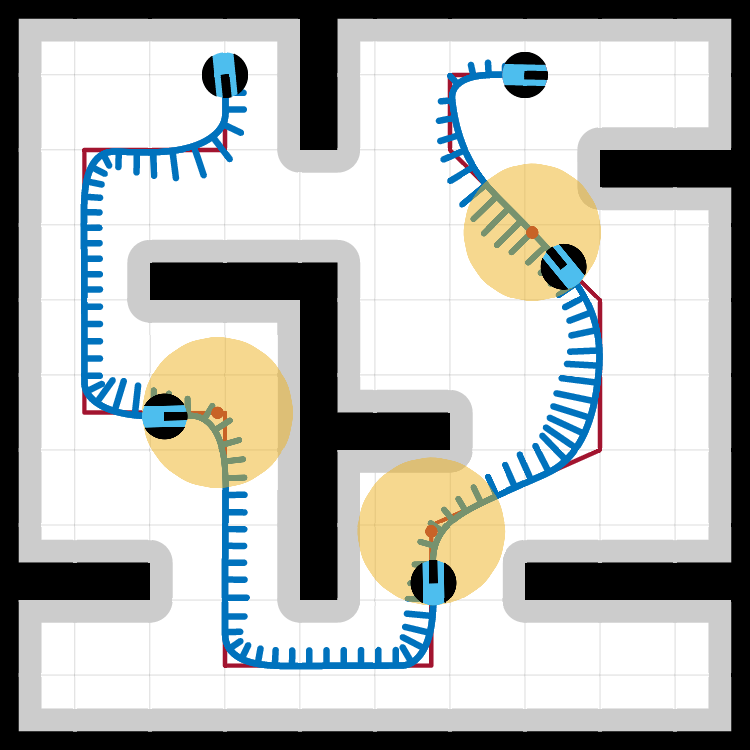} &  
\includegraphics[width = 0.48\columnwidth]{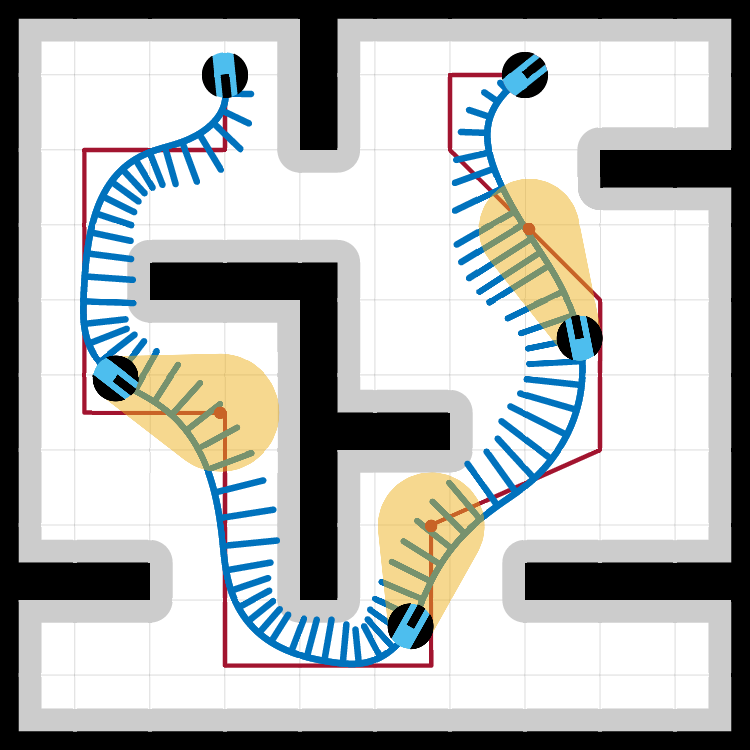} &  
\includegraphics[width = 0.48\columnwidth]{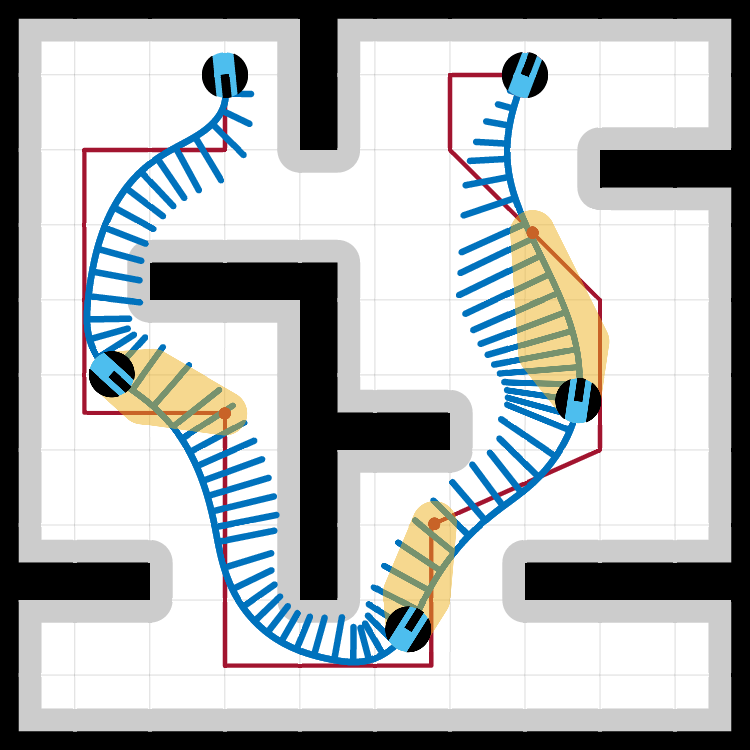} &  
\includegraphics[width = 0.48\columnwidth]{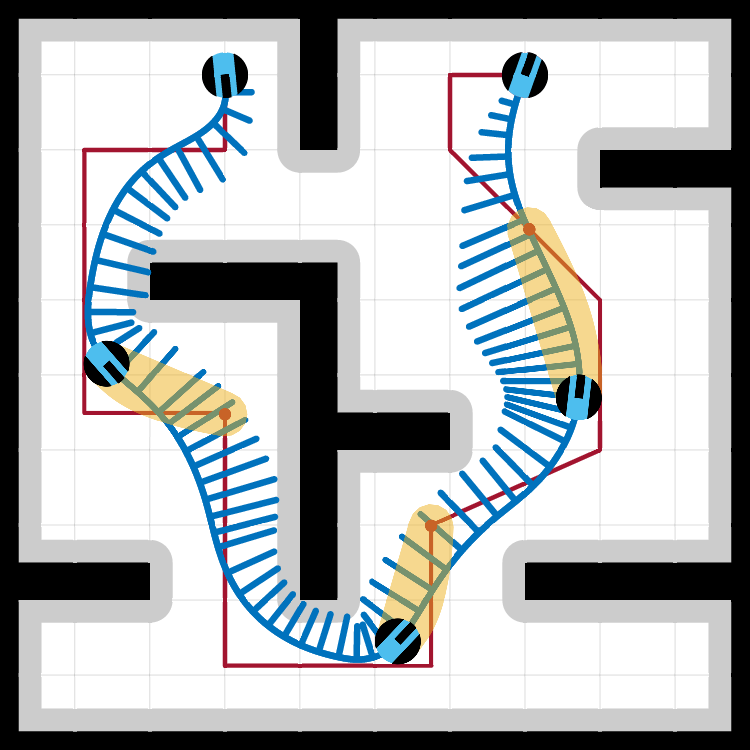} 
\\
\includegraphics[width = 0.48\columnwidth]{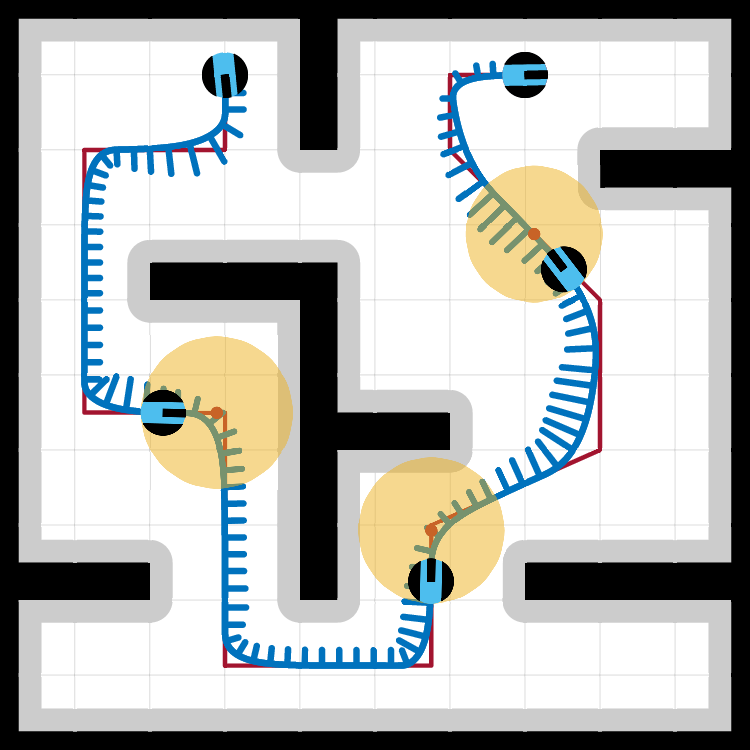} &  
\includegraphics[width = 0.48\columnwidth]{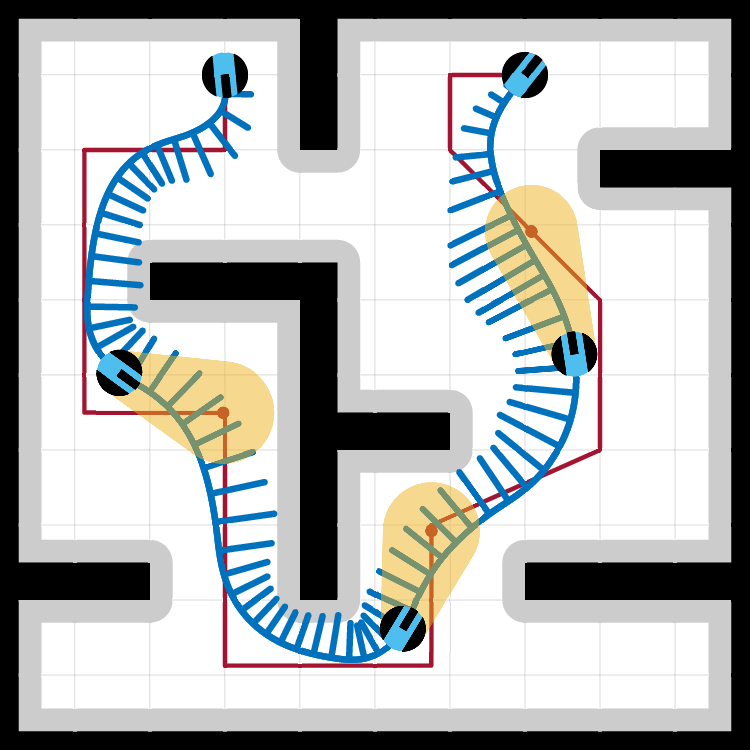} &  
\includegraphics[width = 0.48\columnwidth]{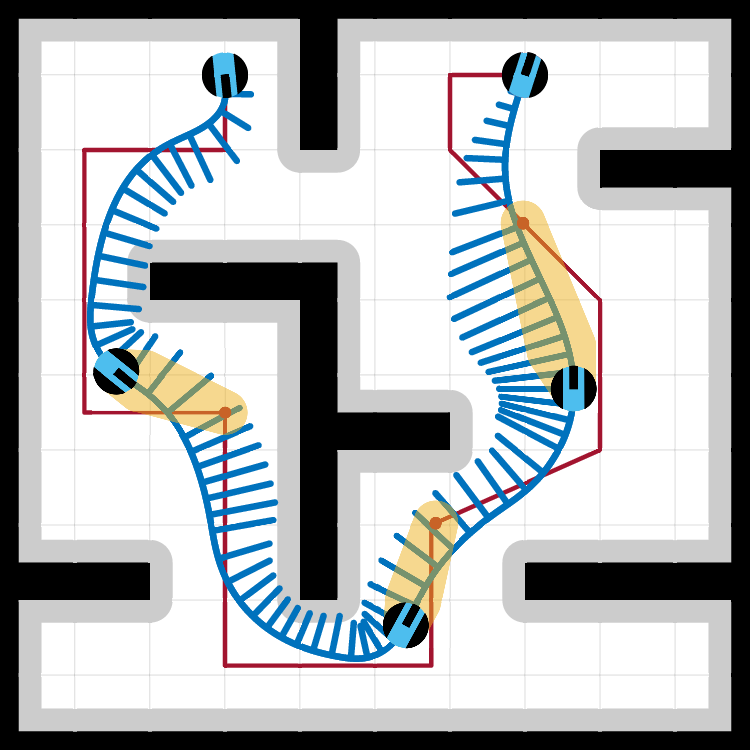} &  
\includegraphics[width = 0.48\columnwidth]{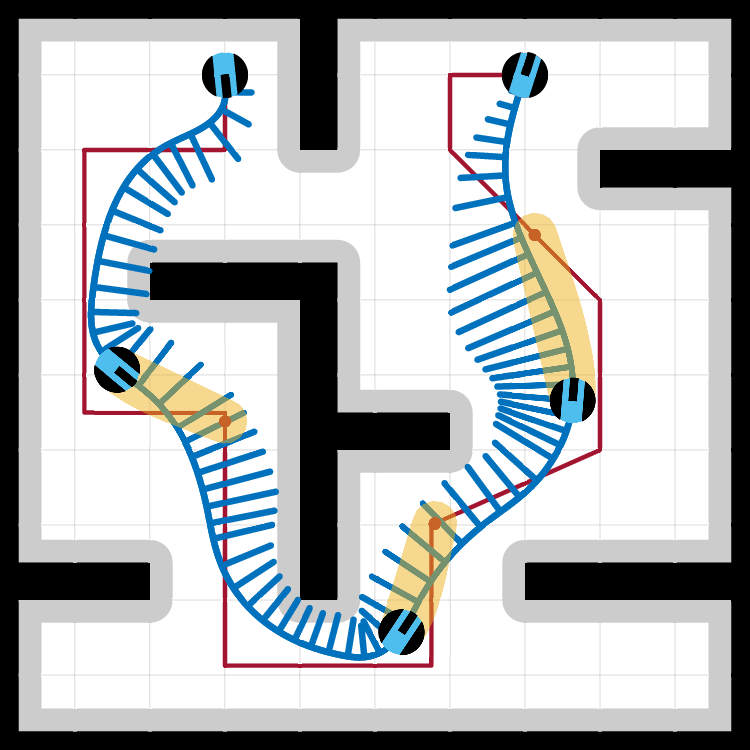} 
\\[-1.5mm]
\footnotesize{(a)} & \footnotesize{(b)} & \footnotesize{(c)} & \footnotesize{(d)}
\end{tabular}
\vspace{-2.0mm}
\caption{Time-governed safe unicycle path-following control around obstacles using feedback motion prediction. 
The safety of the unicycle motion is constantly verified using (a) ball-shaped, (b) cone-shaped, (c) diamond-shaped, (c) forward-reachable motion prediction methods. The unicycle robot motion is illustrated by blue lines, where blue bars indicate robot speed. Yellow regions show an instance of the feedback motion prediction during the robot motion towards a moving
point (red point) along a given reference path (red line).
Here, we set a shared linear control gain of $\lingain=1$ for all simulations and consider two different values for the angular gain as  (top) $\anggain=2$ and (bottom) $\anggain = 3$. 
}
\vspace{-2.5mm}
\label{fig.Safe_Path_Following_Motion}
\end{figure*}

\subsection{Time-Governed Safe Unicycle Path Following}

Consider a disk-shaped unicycle robot of radius $\radius \geq 0$, centered at position $\pos \in \workspace$ with orientation $\ort \in [-\pi, \pi)$, that operates in a known static bounded environment $\workspace \subseteq \R^2$ with a collection of obstacles represented by an open set $\obstspace \subseteq \R^2$.
Hence, the robot's free space, denoted by $\freespace$, of collision-free positions is given by
\begin{align}
    \freespace := \clist{\pos \in \workspace | \ball(\pos, \radius) \subseteq \workspace \setminus \obstspace}
\end{align}
where $\ball(\pos, \radius) := \clist{\vect{y} \in \R^2 | \norm{\vect{y} - \pos} \leq \radius}$ is the closed Euclidean ball centered at $\pos \in \R^2$ with radius $\radius \geq 0$, representing the robot's body.
Let $\path(\pathparam): [\smin, \smax] \rightarrow \freespace$ be a Lipschitz-continuous, collision-free reference path connecting a specified pair of start and goal positions $\startpos, \goalpos \in \freespace$ such that $\path(\smin) = \startpos$ and $\path(\smax) = \goalpos$ and it has a positive clearance from the free space boundary $\partial \freespace$, i.e., $\min_{\pathparam \in [\smin, \smax], \pos \in \partial \freespace} \norm{\path(\pathparam) - \pos} > 0$.

Starting at $t \!=\! 0$ with the initial path parameter \mbox{$\pathparam(0)\! =\! \smin$}, the initial unicycle position $\pos(0) \!= \!\startpos$, and any initial unicycle orientation \mbox{$\ort(0) \in [-\pi, \pi)$}, we design a safe unicycle path-following controller with online continuous time parametrization, using the unicycle motion controller $\ctrl_{\path(\pathparam)}(\pos, \ort) \!=\! \plist{\linvel_{\path(\pathparam)}(\pos, \ort), \angvel_{\path(\pathparam)}(\pos, \ort)\!}$  in \refeq{eq.Bidirectional_Unicycle_Control} towards the path point $\path(\pathparam)$ and an associated feedback motion prediction method $\motionset_{\ctrl_{\path(\pathparam)}}\!(\pos, \ort)$ from \refsec{sec.Unicycle_Motion_Range_Prediction}, as %
\begin{subequations}\label{eq.SafePathFollowing}
\begin{align}
    \dot{\pathparam} &=  \min \plist{ \gain_\safelevel \safedist_\freespace \plist{\motionset_{\ctrl_{\path(\pathparam)\!}}\!(\pos, \ort)} , \!-\gain_\pathparam (\pathparam \!-\! \smax) }\!\!
    \\
    \dot{\pos} &= \linvel_{\path(s)}(\pos, \ort) 
    \\
    \dot{\ort} &= \angvel_{\path(s)}(\pos, \ort)
\end{align}
\end{subequations}
where  $\gain_\safelevel, \gain_\pathparam > 0$ are fixed positive control coefficients, and the safety of the unicycle motion is measured by the minimum distance between the feedback motion prediction set $\motionset_{\ctrl_{\path(\pathparam)}}(\pos, \ort)$  and the free space boundary $\partial \freespace$  as
\begin{align*}\label{eq.safedist}
    \safedist_{\freespace} (\motionset_{\ctrl_{\path(\pathparam)}}\!(\pos, \ort) \!) 
    & \!\ldf\! \!\left \{ \begin{array}{@{}l@{\,}l@{}}
    \min\limits_{\substack{\vect{a} \in \motionset_{\ctrl_{\path(\pathparam)}}\!\!(\pos, \ort)\\ \vect{b} \in \partial \freespace} }\!\!\!\!\! \!\!\!\!\norm{\vect{a} \!-\! \vect{b}} & \text{, if } \motionset_{\ctrl_{\path(\pathparam)}} \!(\pos, \ort) \!\subseteq\! \freespace \\
    0 & \text{, otherwise.}
    \end{array}
    \right.
\end{align*}    
The safe path following dynamics in \refeq{eq.SafePathFollowing} incrementally increase the path parameter $\pathparam$, based on the safety of the predicted unicycle motion until reaching the end of the path, while the unicycle robot under the feedback motion control $\ctrl_{\path(\pathparam)}$ chases the current reference path point $\path(\pathparam)$ as a local goal.
Since the reference path $\path$ is assumed to have a positive clearance from collisions, the asymptotic radial decay property of the feedback motion prediction guarantees that the path parameter $\pathparam(t)$ and the unicycle robot position $\pos(t)$ under the safe path following controller in \refeq{eq.SafePathFollowing} asymptotically converge to the end of the reference path while also guaranteeing that the unicycle robot stays away from collisions along the way \cite{arslan_arXiv2022}, i.e.,
\begin{align*}
\pos(t) \in \freespace \quad \forall t \geq 0, \,\,
\lim_{t\rightarrow \infty} \pathparam(t) = \smax, \,\,
\lim_{t\rightarrow \infty} \pos(t) = \path(\smax).
\end{align*}

\begin{figure}[t]
\centering
\begin{tabular}{@{}c@{\hspace{0.5mm}}c@{}}
\includegraphics[width = 0.49\columnwidth]{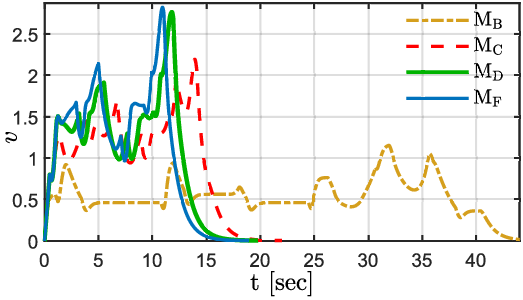} &
\includegraphics[width = 0.49\columnwidth]{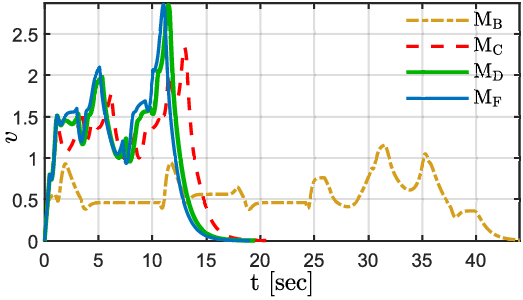} 
\end{tabular}
\vspace{-2.5mm}
\caption{\!\!Unicycle speed profiles during safe path following around obstacles using different unicycle feedback motion prediction methods (ball-shaped $\mathrm{M_B}$, cone-shaped $\mathrm{M_C}$, diamond-shaped $\mathrm{M_D}$, and forward-reachable $\mathrm{M_F}$). These profiles are presented for a shared linear gain of $\lingain \!\!=\! 1$ and two different angular gains: (left) $\!\anggain \!\!= \!2$ and (right) $\!\anggain \!\!=\! 3$. Higher angular gain~leads to faster robot motion and a smaller diamond-shaped~motion~prediction.
}
\label{fig.Safe_Path_Following_Speed}
\vspace{-3mm}
\end{figure}

\subsection{Numerical Simulations}
\label{sec.Numerical_Simulations}

In this part, we provide example numerical simulations\reffn{fn.Numerical_Simulations} to demonstrate safe path following of a unicycle mobile robot around obstacles using feedback motion prediction.
In \reffig{fig.Safe_Path_Following_Motion} and \reffig{fig.Safe_Path_Following_Speed}, we present the resulting unicycle position trajectories and velocity profiles during safe unicycle path following using the ball-shaped, cone-shaped, diamond-shaped, and forward-reachable feedback motion prediction methods of the unicycle motion control in \refeq{eq.Bidirectional_Unicycle_Control}.
As a ground truth, we use the forward-reachable motion set of the unicycle motion control that is numerically computed.
As seen \reffig{fig.Safe_Path_Following_Speed}, the accuracy of feedback motion prediction influences the resulting unicycle motions, leading to significant variations in both speed and travel time.%
\addtocounter{footnote}{1}
\footnotetext{\label{fn.Numerical_Simulations}For all simulations, unless specified, we set linear and angular control gains as $\lingain \!=\! 1$ and $ \anggain \!=\! 2$ for the unicycle control, and the control coefficients for the time governor in \refeq{eq.SafePathFollowing} $\gain_\pathparam \!=\! 4$, $\gain_\safelevel \!=\! 4$. We use the arc-length parametrization of a given reference path $\path(\pathparam)$ such that the reference path length $L$ determines the path parameter range as $[\smin, \smax] = [0, L]$. All simulations are obtained by numerically solving the time-governed unicycle path-following dynamics in  \refeq{eq.SafePathFollowing} using the \texttt{ode45} function of MATLAB. Please see the accompanying video for the animated robot motion.} 
As expected, the forward-reachable motion prediction method shows superior performance in terms of average speed and travel time, although this comes at a significantly higher computational cost.
In addition to the numerical computation of the forward-reachable motion set, computing the distance-to-collision at each point within the forward-reachable motion set for safety assessment is computationally demanding.
On the other hand, the diamond-shaped unicycle motion prediction demonstrates comparable performance like the forward-reachable motion prediction at a significantly lower computation cost because of its simple triangular shape and explicit analytical form  in \refeq{eq.Diamond-Shaped_Motion_Prediction}.
The conic motion prediction also exhibits reasonable performance at a similar computational cost to the diamond-shaped motion prediction. However, it is relatively less accurate as it depends on the unicycle state but has no direct dependency on control parameters.
The circular unicycle motion prediction results in the slowest motion because it is the most conservative and less accurate compared to other unicycle motion predictions, relying solely on the unicycle's distance to the goal.
Overall, feedback motion prediction that strongly depends on the robot's state and control parameters can more accurately capture the closed-loop robot motion, enabling faster safe robot motion around obstacles.

\section{Conclusions}
\label{sec.Conclusions}

In this paper, we introduce a highly simple, highly accurate triangular feedback motion prediction method for a standard unicycle motion control approach with angular feedback linearization.  
We achieve this by explicitly determining the total turning effort and the final orientation of the unicycle control, enabling us to build an intuitive geometric characterization of the closed-loop unicycle motion.
We also present circular and conic feedback motion prediction methods based on other important geometric properties of the unicycle control, such as decreasing the positional goal distance and the orientational goal alignment distance.
In addition to mathematically demonstrating the superior accuracy of the triangular motion prediction over the circular and conic alternatives (\refprop{prop.Inclusion_Order}), we showcase and compare example numerical applications of these feedback motion prediction methods for safe path following around obstacles. 
We observe that the strong dependency of the triangular feedback motion prediction on the unicycle state and control parameters yields a comparable performance as the exact forward-reachable motion set of the unicycle control at a significantly lower computational cost.
This makes the triangular feedback motion prediction the most suitable method for real-time safety-critical navigation applications of unicycle mobile robots.

Our current work in progress focuses on perception-aware safe unicycle motion control with real hardware experiments, especially for safe robot navigation in unknown dynamic environments \cite{arslan_koditschek_IJRR2019}. 
Another promising research direction is the use of feedback motion prediction in model predictive control and sampling-based motion planning \cite{arslan_berntorp_tsiotras_ICRA2017}.

\bibliographystyle{IEEEtran}
\bibliography{references}

\appendices 

\section{Proofs}
\label{app.Proofs}

\subsection{Proof of \reflem{lem.Global_Convergence}}
\label{app.Global_Convergence}

\begin{proof}
By considering $V_{\goal}(\pos, \ort) = \headingerror_{\goal}(\pos, \ort)^2 + \norm{\goal \!-\! \pos}^2$ as a Lyapunov function candidate, one can verify using \refeq{eq.Distance_To_Goal_Dynamics} and \refeq{eq.Linear_Angular_Heading_Error_Dynamics} that for any $\pos \neq \goal$ 
\begin{align*}
\dot{V}_{\goal}(\pos, \ort) = -2\anggain \headingerror_{\goal}(\pos, \ort)^2 - 2 \lingain \plist{\ovecTsmall{\ort}\! \!(\goal\!-\!\pos)\!\!}^{\!2} \!< 0  
\end{align*}
since $\headingerror_{\goal}(\pos, \ort) = 0$ implies \mbox{$\plist{\!\ovecTsmall{\ort}\! \!(\goal\!-\!\pos)\!\!}^{\!2}\!\! =\! \norm{\goal \! - \! \pos}^2$}.
Hence, the result follows from the LaSalle's invariance principle \cite{khalil_NonlinearSystems2001}.
\end{proof}

\subsection{Proof of \refprop{prop.Total_Turning_Effort}}
\label{app.Total_Turning_Effort}

\begin{proof}
The unicycle control in \refeq{eq.Bidirectional_Unicycle_Control} results in the linear heading error dynamics $\dot{\headingerror}_{\goal}(\pos, \ort) = - \anggain \headingerror_{\goal}(\pos, \ort)$ in \refeq{eq.Linear_Angular_Heading_Error_Dynamics}  whose analytical solution along the closed-loop unicycle trajectory $(\pos(t), \ort(t))$ is given for $t \geq 0$ by
\begin{align}
\headingerror_{\goal}(\pos(t), \ort(t)) = \headingerror_{\goal}(\pos_0, \ort_0) e^{-\anggain t}. 
\end{align}
This leads to the following angular velocity profile in \refeq{eq.Bidirectional_Unicycle_Control}:
\begin{align}
\angvel_{\goal}(\pos(t), \ort(t)) & =  \anggain  \headingerror_{\goal}(\pos_0, \ort_0) e^{-\anggain t} \nonumber
\\
& \quad \quad +  \frac{\lingain}{2} \sin\plist{2\headingerror_{\goal}(\pos_0, \ort_0) e^{-\anggain t}} .
\end{align} 
Therefore, the total turning effort can be calculated using integration by substitution (with $u = 2 \headingerror_{\goal}(\pos_0, \ort_0) e^{-\anggain t}$) 
as
\begin{subequations}
\begin{align}
\totalturning_{\goal}(\pos_0, \ort_0) \hspace{-8mm}& \hspace{+8mm}= \int_{0}^{\infty} \angvel_{\goal}(\pos(t), \ort(t)) \diff t 
\\
& = \int_{0}^{\infty}  \anggain  \headingerror_{\goal}(\pos_0, \ort_0) e^{-\anggain t} \diff t \nonumber \\
& \quad \quad + \int_{0}^{\infty} \frac{\lingain}{2} \sin\plist{2\headingerror_{\goal}(\pos_0, \ort_0) e^{-\anggain t}} \diff t
\\
& = \headingerror_{\goal}(\pos_0, \ort_0) + \frac{\lingain}{2 \anggain}\int_{0}^{2 \headingerror_{\goal}(\pos_0, \ort_0)} \frac{\sin(u)}{u} \diff u \!\!
\\
& = \headingerror_{\goal}(\pos_0, \ort_0) + \frac{\lingain}{2 \anggain} \Si\plist{2 \headingerror_{\goal}(\pos_0, \ort_0)}
\end{align}
\end{subequations}
where the substitution variable $u = 2 \headingerror_{\goal}(\pos_0, \ort_0) e^{-\anggain t}$ satisfies $\diff u = -2 \anggain \headingerror_{\goal}(\pos_0, \ort_0) e^{-\anggain  t}  \diff t = - \anggain u \diff t$, $\diff t = -\frac{1}{\anggain  u} \diff u$, $u(0) = 2 \headingerror_{\goal}(\pos_0, \ort_0)$ and  $u(\infty) = 0$. 
\end{proof}

\subsection{Proof of \reflem{lem.Projected_Goal}}
\label{app.Projected_Goal}

\begin{proof}
The result follows from the triangular geometry as follows: $\cpos_{\goal}(\pos, \ort)$ is located away from $\pos$ by a distance of  $\cos(\headingerror_{\goal}\!(\pos,\ort)\!) \!\norm{\goal\! - \pos}$ along the vector $\Rmat(-\headingerror_{\goal}\!(\pos,\ort)\!) \!(\goal \!- \pos)$, whereas $\cposr_{\goal}(\pos, \ort)$ has the same distance from $\pos$ along the reflected direction $\Rmat(+\headingerror_{\goal}(\pos,\ort)\!) (\goal\! - \pos)$.
\end{proof}

\subsection{Proof of \reflem{lem.Heading_Line_Intersection}}
\label{app.Heading_Line_Intersection}

%
\begin{proof}
For $\lingain \leq \anggain$, the total turning effort  is bounded in \refeq{eq.Turning_Effort_Heading_Error_Inequality} as $\absval{\totalturning_{\goal}(\pos, \ort)} < \pi$ and so the initial and final heading lines always intersect.
The current and final heading errors and the total turning effort also satisfy \refeq{eq.Total_Turning_Initila_Final_Heading_Error_Equality}.
%
Hence, the unicycle position $\pos$, the goal location $\goal$, and the intersection point $\xpos_{\goal}(\pos, \ort)$ define a triangle with corresponding interior angles $\absval{\headingerror_{\goal}(\pos,\ort)}$, $\absval{\headingerror_{\goal}(\pos, \fort_{\goal}(\pos, \ort))}$, and $\pi - \absval{\totalturning_{\goal}(\pos, \ort)}$, respectively.  
Due to the sine theorem, the intersection point $\xpos_{\goal}(\pos, \ort)$ of the current and final heading lines satisfies
{\small
\begin{align*}
\scalebox{0.95}{$
\dfrac{\norm{\xpos_{\goal}(\pos, \ort) - \pos}}{\absval{\sin(\headingerror_{\goal}(\pos, \fort_{\goal}(\pos, \ort)\!)\!)}} = \dfrac{\norm{\xpos_{\goal}(\pos, \ort) - \goal}}{\absval{\sin(\headingerror_{\goal}(\pos,\ort)\!)}} = \dfrac{\norm{\pos - \goal}}{\absval{\sin(\totalturning_{\goal}(\pos, \ort)\!)}}
$}
\end{align*} 
}%
and it follows from  the definition of the angular heading error $\headingerror_{\goal}(\pos,\ort)$ in \refeq{eq.Angular_Heading_Error} that

\noindent
{\small
\begin{align*}
\frac{\xpos_{\goal}(\pos, \ort) - \pos}{\norm{\xpos_{\goal}(\pos, \ort) - \pos}} &= \Rmat(-\headingerror_{\goal}(\pos, \ort)\!) \frac{\goal - \pos}{\norm{\goal - \pos}}
\\
\frac{\xpos_{\goal}(\pos, \ort) - \goal}{\norm{\xpos_{\goal}(\pos, \ort) - \goal}} &= -\Rmat(-\headingerror_{\goal}(\pos, \fort(\pos, \ort)\!)\!) \frac{\goal - \pos}{\norm{\goal - \pos}}.
\end{align*}
}%
Hence, one can determine the location of the intersection point $\xpos_{\goal}(\pos, \ort)$ using these observations as in  \refeq{eq.Heading_Line_Intersection}, whereas its reflection point $\xposr_{\goal}(\pos, \ort)$ can be obtained by changing the direction of rotation of the goal direction $\goal - \pos$. 
\end{proof}

\subsection{Proof of \refprop{prop.Triangular_Motion_Bound}}
\label{app.Triangular_Motion_Bound}

\begin{proof}
The result trivially holds for $\pos_0 = \goal$ since $\xpos_{\goal}(\pos_0, \ort_0) \in \blist{\pos_0, \goal}$. Hence, for the rest of the proof, we assume the unicycle is away from the goal, i.e., $\pos_0 \neq \goal$.

To observe that the unicycle position $\pos(t)$ along the closed-loop unicycle trajectory $(\pos(t), \ort(t))$  stays in $\conv(\pos_0, \goal, \xpos_{\goal}(\pos_0, \ort_0))$, we show below that if the unicycle approaches to the boundary of $\conv(\pos_0, \goal, \xpos_{\goal}(\pos_0, \ort_0))$, its velocity vector $\dot{\pos}(t)$ always points towards a point in $\conv(\pos_0, \goal, \xpos_{\goal}(\pos_0, \ort_0))$. Therefore, the unicycle position $\pos(t)$ remains in $\conv(\pos_0, \goal, \xpos_{\goal}(\pos_0, \ort_0))$ because of the Nagumo sub-tangentiality condition \cite{blanchini_Automatica1999}.   

Recall from \refeq{eq.Linear_Angular_Heading_Error_Dynamics} that the linear angular heading dynamics ensure a monotonically decaying angular heading error, i.e.,
\begin{align*}
\absval{\headingerror_{\goal}(\pos(t), \ort(t))} \geq \absval{\headingerror_{\goal}(\pos(t'), \ort(t'))} \quad \forall t'\geq t \geq 0.
\end{align*}
Moreover, the linear velocity control is defined in \refeq{eq.Bidirectional_Unicycle_Control} as
\begin{align*}
\dot{\pos} &= \linvel_{\goal}(\pos, \ort) \ovectsmall{\ort} = \lingain \ovecTsmall{\ort}(\goal - \pos) \ovectsmall{\ort} 
\\
& = \lingain \cos(\headingerror_{\goal}(\pos, \ort)) \Rmat(-\headingerror_{\goal}(\pos,\ort)) (\goal - \pos)
\end{align*}
which is due to the following relations
{\small
\begin{align*}
\absval{\ovecTsmall{\ort}(\goal - \pos)} &=  \cos(\headingerror_{\goal}(\pos, \ort)) \norm{\pos - \goal}   \\
\frac{\dot{\pos}}{\norm{\dot{\pos}}} &= \Rmat(-\headingerror_{\goal}(\pos, \ort)\!) \frac{\goal - \pos}{\norm{\goal - \pos}}.
\end{align*}
}%
So the unicycle moves in the direction of $\scalebox{0.89}{$\Rmat(\!-\headingerror_{\goal}\!(\pos,\!\ort )\!) (\goal\!\! - \! \pos)$}$.

Accordingly, one can observe as below that if the unicycle reaches the boundary of $\conv(\pos_0, \goal, \xpos_{\goal}(\pos_0, \ort_0))$, it moves towards a point within that convex hull.

$\bullet$ If $\pos(t) = \goal$, then $\dot{\pos}(t) \!=\! 0$ and $\pos(t') \!=\! \goal$ for all $t'\geq t$.

$\bullet$ If $\pos(t) \in [\xpos_{\goal}(\pos_0, \ort_0), \goal)$, then we have 
\begin{align*}
\headingerror_{\goal}(\pos(t), \fort_{\goal}(\pos(t), \ort(t)\!)\!) &= \headingerror_{\goal}(\pos(t), \fort_{\goal}(\pos_0, \ort_0)\!) \\
&  = 0 = -\frac{\lingain}{2\anggain} \Si(\headingerror_{\goal}(\pos(t), \ort(t)\!)\!)
\end{align*}
which implies the unicycle has zero angular heading error, i.e.,  $\headingerror_{\goal}(\pos(t), \ort(t)\!) \!= \!0$,
 and so it moves towards the goal $\goal$ since  $ \dot{\pos} \,  \sim \, \Rmat(-\headingerror_{\goal}(\pos(t),\ort(t)\!)\!) (\goal - \pos(t)\!) = \goal - \pos(t)$.

$\bullet$ If $\pos(t) \in [\pos_0, \xpos_{\goal}(\pos_0, \ort_0)]$ or $\pos(t) \in [\pos_0, \goal)$, then  the unicycle heading error satisfies 
\begin{align}
\absval{\headingerror_{\goal}(\pos(t), \ort(t)\!)} & \leq  \absval{\headingerror_{\goal}(\pos_0, \ort_0)} \nonumber
\\
&\leq \absval{\measuredangle (\xpos_{\goal}(\pos_0, \ort_0)\! -\! \pos(t) , \goal\! -\! \pos(t)\!)} \label{eq.Angular_Triangle_Geometry}
\end{align}
where all these angles share the same sign, and we denote by  $\measuredangle(\vect{u}, \vect{v}):= \arctan\plist{\!\tr{\plist{\Rmat(\tfrac{\pi}{2}) \vect{u}}\!} \vect{v}\Big/\tr{\vect{u}} \vect{v}\!}$ the counterclockwise angle from $\vect{u}$ to $\vect{v}$, e.g., $\headingerror_{\goal}(\pos, \ort) = \measuredangle( \scalebox{0.85}{$\ovectsmall{\ort}$}, \goal - \pos)$. 
As a result, the unicycle velocity $\dot{\pos}$ points in the direction of $\Rmat(-\headingerror_{\goal}(\pos(t),\ort(t))) (\goal - \pos)$ from $\pos(t)$  to a point between the intersection point $\xpos_{\goal}(\pos_0, \ort_0)$ and the goal $\goal$, within the set  $\conv(\pos_0, \goal, \xpos_{\goal}(\pos_0, \ort_0))$.

Therefore, the unicycle always moves towards a point within $\conv(\pos_0, \goal, \xpos_{\goal}(\pos_0, \ort_0))$, which in turn defines a bound on the unicycle position trajectory $\pos(t)$ for all future times $t \geq 0$, which completes the proof.
\end{proof}

\subsection{Proof of \refprop{prop.Circular_Conic_Motion_Prediction}}
\label{app.Circular_Conic_Motion_Prediction}

\begin{proof}
%
The circular motion prediction set results from the continuously decreasing Euclidean distance between the unicycle position and the goal position, as described in \refeq{eq.Distance_To_Goal_Dynamics}, along the closed-loop unicycle motion.

The conic motion prediction set can be verified by using its representation as a union of a triangle and a ball as
\begin{align}
\!\!\motionset_{\goal, \cone}(\pos_0, \ort_0) =  &\, \conv(\pos_0, \cpos_{\goal}(\pos_0, \ort_0), \cposr_{\goal}(\pos_0, \ort_0) ) \nonumber \\
& \,\, \cup \ball(\goal\!, \sin(\absval{\headingerror_{\goal}(\pos_0, \ort_0)}) \norm{\goal \!- \!\pos_0})\!\!\!
\end{align}
where $\cpos_{\goal}(\pos, \ort)$ is the projected goal point onto the current heading line that bounds the motion cone  $\motionset_{\goal, \cone}(\pos, \ort)$, and $\cposr_{\goal}(\pos, \ort)$ is its reflection with respect to the goal line (i.e., the cone center line), as defined in \refeq{eq.Projected_Goal}.
We have from the circular motion prediction that if the unicycle enters $\ball(\goal\!, \sin(\absval{\headingerror_{\goal}(\pos_0, \ort_0)}) \norm{\goal \!- \!\pos_0})$, then it remains inside that circular motion set.
Also, as in the proof of \refprop{prop.Triangular_Motion_Bound} and also described in \cite{isleyen_vandewouw_arslan_IROS2023}, one can observe that if the unicycle position is in $\conv(\pos_0, \goal, \cpos_{\goal}(\pos, \ort)) \setminus \ball(\goal\!, \sin(\absval{\headingerror_{\goal}(\pos_0, \ort_0)}) \norm{\goal \!- \!\pos_0})$, then the unicycle moves towards a point in between $\goal$ and $\cpos_{\goal}(\pos, \ort)$. Therefore, the unicycle position trajectory $\pos(t)$ remains in $\conv(\pos_0, \goal, \cpos_{\goal}(\pos, \ort))$ until it reaches the ball $\ball(\goal\!, \sin(\absval{\headingerror_{\goal}(\pos_0, \ort_0)}) \norm{\goal \!- \!\pos_0})$ that contains the unicycle position trajectory for all the remaining future times. 
Hence, the result follows since $\motionset_{\goal, \cone}(\pos, \ort)$ is a superset of $\conv(\pos, \goal, \cpos_{\goal}(\pos, \ort)) \cup \ball(\goal\!, \sin(\absval{\headingerror_{\goal}(\pos, \ort)}) \norm{\goal \!- \!\pos})$.
\end{proof}

\subsection{Proof of \refprop{prop.Positive_Inclusion}}
\label{app.Positive_Inclusion}

\begin{proof}
The positive inclusion of the ball-shaped motion prediction set $\motionset_{\ctrl_{\goal}, \ball}(\pos, \ort)$ follows from that the distance to the goal, $\norm{\pos - \goal}$, is decreasing over time as shown in~\refeq{eq.Distance_To_Goal_Dynamics}.

The positive inclusion of the cone-shape motion prediction set $\motionset_{\ctrl_{\goal}, \cone}(\pos, \ort)$ is due to its convex hull property,
\begin{align*}
\motionset_{\ctrl_{\goal}, \cone}(\pos, \ort) = \conv\plist{\pos, \ball(\goal, \sin(\absval{\headingerror_{\goal}(\pos,\ort)} \norm{\pos - \goal})\!)\!}
\end{align*}  
since both the absolute angular heading error and the distance to the goal are decreasing along the closed-loop unicycle trajectory, as shown in \refeq{eq.Angular_Heading_Error_Dynamics} and \refeq{eq.Distance_To_Goal_Dynamics}, respectively. 

The positive inclusion of the diamond shaped motion prediction  $\motionset_{\ctrl_{\goal}, \dshape}(\pos, \ort) = \conv\plist{\pos, \goal, \xpos_{\goal}(\pos, \ort), \xposr_{\goal}(\pos, \ort)}$ along the closed-loop unicycle trajectory $(\pos(t), \ort(t))$ is due to its convexity and the fact that  for all $t' \geq t$ we have  $\pos(t'), \goal \in \motionset_{\ctrl_{\goal}, \dshape}(\pos(t), \ort(t))$ and  $\xpos_{\goal}(\pos(t'), \ort(t'))$ is in between $ \xpos_{\goal}(\pos(t), \ort(t))$ and $\goal$, which also, by symmetry, holds for its reflection, i.e., $\xposr_{\goal}(\pos(t'), \ort(t')) \in \blist{\xposr_{\goal}(\pos(t), \ort(t)), \goal}$.   
Here, one can conclude $\xpos_{\goal}(\pos(t'), \ort(t')) \in \blist{\xpos_{\goal}(\pos(t), \ort(t)), \goal}$ for all $t'\geq t$ using  the geometry of the triangle  $\conv(\pos, \goal, \xpos_{\goal}(\pos,\ort))$ as in the proof of \refprop{prop.Triangular_Motion_Bound} using \refeq{eq.Angular_Triangle_Geometry}. 
Because the final heading line is constant, $\conv(\pos, \goal, \xpos_{\goal}(\pos,\ort))$ is a triangular motion bound on the closed-loop unicycle position trajectory, i.e., $\pos(t') \in \conv(\pos(t), \goal, \xpos_{\goal}(\pos(t), \ort(t)))$ for all $t'\geq t$, and the absolute angular heading error is decreasing as in \refeq{eq.Angular_Heading_Error_Dynamics}, i.e., $\absval{\headingerror_{\goal}\!(\pos(t), \ort(t)\!)} \! \geq\! \absval{\headingerror_{\goal}\!(\pos(t'), \ort(t')\!)}$ for all $t'\geq t$.  

Last, the positive inclusion of the forward-reachable motion set $\motionset_{\ctrl_{\goal}, \fwdsim}$
is directly evident from \refdef{def.Forward_Simulated_Motion_Prediction}. 
\end{proof}

\subsection{Proof of \refprop{prop.Radial_Decay}}
\label{app.Radial_Decay}

\begin{proof}
The radii of the ball-shaped, cone-shaped, and forward-reachable motion prediction sets relative to the goal are all equal and  given by the unicycle's position distance to the goal $\norm{\pos - \goal}$, i.e., 
\begin{align*}
\max_{\pos' \in \motionset_{\ctrl_{\goal}, \ball}(\pos, \ort)} \norm{\pos' \!- \! \goal} &\!=\! \max_{\pos' \in \motionset_{\ctrl_{\goal}, \cone}(\pos, \ort)} \norm{\pos' \!- \! \goal} 
\\
& \!=\! \max_{\pos' \in \motionset_{\ctrl_{\goal}, \fwdsim}(\pos, \ort)} \norm{\pos' \!- \! \goal} \!= \norm{\pos \!-\! \goal}
\end{align*}
which asymptotically decreases to zero as shown in \refeq{eq.Distance_To_Goal_Dynamics}. 

The radius of the diamond-shaped motion prediction set $\motionset_{\ctrl_{\goal}, \dshape}(\pos, \ort) \! =\! \conv\plist{\pos, \goal, \xpos_{\goal}(\pos, \ort), \xposr_{\goal}(\pos, \ort)\!}$ is given by 
\begin{align*}
\max_{\pos' \in \motionset_{\ctrl_{\goal}, \dshape}(\pos, \ort)} \norm{\pos' \!- \! \goal} = \max(\norm{\pos - \goal}, \norm{\xpos_{\goal}(\pos, \ort) - \goal}) 
\end{align*}
where $\norm{\pos - \goal}$ asymptotically decreases to zero as shown in \refeq{eq.Distance_To_Goal_Dynamics}, and  $\norm{\xpos_{\goal}(\pos, \ort) - \goal}$ becomes bounded above by $\norm{\pos - \goal}$ in finite time%
\footnote{As discussed in the proof of \reflem{lem.Heading_Line_Intersection}, we have from the sine theorem 
\begin{align*}
\norm{\xpos_{\goal}(\pos, \ort) - \goal} = \frac{\sin(\headingerror_{\goal}(\pos, \ort))}{\sin(\totalturning_{\goal}(\pos, \ort))}\norm{\pos - \goal}
\end{align*}
and $\absval{\headingerror_{\goal}(\pos, \ort)} \leq \absval{\totalturning_{\goal}(\pos, \ort)} \leq (1 + \tfrac{\lingain}{\anggain})\absval{\headingerror_{\goal}(\pos, \ort)}$. This implies due to the monotonicity of the sine function over $[-\pi/2, \pi/2]$ that for $\lingain \leq \anggain$ 
\begin{align*}
\absval{\headingerror_{\goal}(\pos, \ort)} \leq \pi/4 \Longrightarrow \norm{\xpos_{\goal}(\pos, \ort) - \goal} \leq \norm{\pos - \goal}.  
\end{align*}
}
due to the exponential decay of the angular heading error as shown in \refeq{eq.Angular_Heading_Error_Dynamics}. 
Thus, the result follows.
\end{proof}

\subsection{Proof of \refprop{prop.Motion_Prediction_Distance}}
\label{app.Motion_Prediction_Distance}

\begin{proof}
The result follows from the following properties of locally Lipschitz functions: (i) a piecewise continuously differentiable function is locally Lipschitz on its domain \cite{chaney_NA1990}, (ii) metric projection onto convex sets are piecewise continuously differentiable (and so locally Lipschitz) \cite{kuntz_scholtes_JMAA1994}, (iii) a continuous selection and composition of locally Lipschitz continuous functions is also locally Lipschitz \cite{liu_JCO1995}, and (iv) the
minimum distance to a set is Lipschitz continuous under affine transformations of that set \cite{isleyen_vandewouw_arslan_RAL2022}.
For example, the minimum set distance $\min_{\pos' \in \motionset_{\ctrl_{\goal}}(\pos, \ort)} \norm{\vect{y} - \pos'}$ is locally Lipschitz continuous with respect to $\vect{y}$  due to the properties (ii) and (iii) since the Euclidean norm is Lipschitz.

The minimum distance of a point $\vect{y}$ to the ball-shaped motion prediction set $\motionset_{\ctrl_{\goal}, \ball}(\pos, \ort) = \ball(\goal, \norm{\pos - \goal})$  is given by $\min(0, \norm{\vect{y} - \vect{\goal}} - \norm{\pos - \goal})$ and is locally Lipschitz continuous with respect to the unicycle state $(\pos, \ort)$ and the goal position $\goal$ due to the property (iii). 

The minimum distance of a point $\vect{y}$ to the cone-shaped motion prediction set $\motionset_{\ctrl_{\goal}, \cone}(\pos, \ort)$
can be determined using its decomposition into a triangle and a ball in \refeq{eq.Conic_Motion_Prediction_Decomposition} as the minimum distance of $\vect{y}$ to  $\conv(\pos, \goal, \cpos_{\goal}(\pos, \ort), \cposr_{\goal}(\pos, \ort))$ and $\ball(\goal, \norm{\cpos_{\goal}(\pos, \ort) - \goal})$.
Hence, its local Lipschitz continuity, with respect to the unicycle state $(\pos, \ort)$ and the goal position $\goal$, follows from the properties (ii), (iii), and (iv).
Because $\cpos_{\goal}(\pos, \ort)$ and $\cposr_{\goal}(\pos, \ort))$ in \refeq{eq.Projected_Goal} are continuously differentiable for $\pos \neq \goal$ and $\absval{\headingerror_{\goal}(\pos, \ort)}< \tfrac{\pi}{2}$; and otherwise $\lim_{\pos \rightarrow \goal}\cpos_{\goal}(\pos, \ort) =  \lim_{\pos \rightarrow \goal}\cposr_{\goal}(\pos, \ort) = \goal$  and  $\lim_{\absval{\headingerror_{\goal}(\pos, \ort)} \rightarrow \frac{\pi}{2}} \cpos_{\goal}(\pos, \ort) = \lim_{\absval{\headingerror_{\goal}(\pos, \ort)} \rightarrow \frac{\pi}{2}} \cposr_{\goal}(\pos, \ort)= \pos$.    

Similarly, the minimum distance of a point $\vect{y}$ to the diamond-shaped motion prediction set $\motionset_{\ctrl_{\goal}, \dshape}(\pos, \ort) = \conv\plist{\pos, \goal, \xpos_{\goal}(\pos, \ort), \xposr_{\goal}(\pos, \ort)}$ can be determined using its decomposition into two triangles, $\conv\plist{\pos, \goal, \xpos_{\goal}(\pos, \ort)}$ and $\conv\plist{\pos, \goal, \xposr_{\goal}(\pos, \ort)}$.  
Hence, the local Lipschitz continuity of the point distance to  $\motionset_{\ctrl_{\goal}, \dshape}(\pos, \ort)$, with respect to the unicycle state $(\pos, \ort)$ and the goal position $\goal$, is due to the properties (ii), (iii), and (iv).
Because $ \xpos_{\goal}(\pos, \ort)$ and $ \xposr_{\goal}(\pos, \ort)$ in \refeq{eq.Heading_Line_Intersection} are continuously differentiable for $\pos \neq \goal$ and $\absval{\headingerror_{\goal}(\pos, \ort)}< \tfrac{\pi}{2}$; and otherwise, we have $\lim_{\pos \rightarrow \goal}\xpos_{\goal}(\pos, \ort) =  \lim_{\pos \rightarrow \goal}\xposr_{\goal}(\pos, \ort) = \goal$  and 
{\small
\begin{align*}
&\!\lim_{\absval{\headingerror_{\goal}(\pos, \ort)} \rightarrow \tfrac{\pi}{2}} \clist{ \xpos_{\goal}(\pos, \ort), \xposr_{\goal}(\pos, \ort)} \\
& \hspace{8mm}=  \scalebox{0.8}{$\clist{\!\pos \!+\! \dfrac{1}{\cos(\frac{\lingain}{2\anggain} \Si(\pi)\!)}\Rmat(\tfrac{\pi}{2}) (\goal\! -\! \pos), \pos \!-\!  \dfrac{1}{\cos(\frac{\lingain}{2\anggain} \Si(\pi)\!)}\Rmat(\tfrac{\pi}{2}) (\goal \!- \!\pos)\!}$}
\end{align*}
}%
which follows from \refeq{eq.Turning_Effort_Heading_Error_Equality} and the fact that $0\!<\!\frac{\lingain}{2\anggain} \Si(\pi) \!<\! \tfrac{\pi}{2}$ for $0 \!<\! \lingain \!<\! \anggain$. This completes the proof.
\end{proof}

\subsection{Proof of \refprop{prop.Inclusion_Order}}
\label{app.Inclusion_Order}

\begin{proof}
The result follows from the triangle geometry and the sine theorem (see the proof of \reflem{lem.Heading_Line_Intersection}) since  $\xpos_{\goal}(\pos, \ort) = \cpos_{\goal}(\pos, \ort)$ for $\absval{\totalturning_{\goal}(\pos, \ort)} = \tfrac{\pi}{2}$, and $\xpos_{\goal}(\pos, \ort) \in [\pos, \cpos_{\goal}(\pos, \ort)]$ when the absolute total turning effort is less than $\tfrac{\pi}{2}$.
\end{proof}

\section{Directional Unicycle Control \& Prediction}
\label{app.Directional_Unicycle_Control}

In this part, we briefly demonstrate a simple adaptation of the bidirectional unicycle control in \refeq{eq.Bidirectional_Unicycle_Control} for directional unicycle robots (without providing proofs).
Because robots might be restricted to move in either forward or backward direction, possibly due to perception and actuation constraints such as a restricted field of sensing to the front or back, or the presence of a manipulator arm in the front or back.

\subsection{Forward Unicycle Control and Motion Prediction}

As in \refsec{sec.Unicycle_Dynamics_Control}, using angular feedback linearization, we construct the forward unicycle motion control $\fwdctrl_{\goal}(\pos, \ort):=(\fwdlinvel_{\goal}(\pos, \ort), \fwdangvel_{\goal}(\pos, \ort))$ that moves the unicycle robot in the forward direction towards the goal position $\goal$ by setting the linear and angular velocity inputs as 
{\small
\begin{subequations}\label{eq.Forward_Unicycle_Control}
\begin{align}
\fwdlinvel_{\goal}(\pos, \ort) &:= \max \plist{0, \lingain \ovecTsmall{\ort}\! \!(\goal\!-\!\pos)\! \!}  \nonumber
\\
& = \scalebox{0.92}{$\left \{ 
 \begin{array}{@{}c@{}l@{}}
\lingain \ovecTsmall{\ort}\! \!(\goal\!-\!\pos),  & \text{ if } |\fwdheadingerror_{\goal}\!(\pos, \ort)| \! \leq \! \tfrac{\pi}{2} \\
0, & \text{ otherwise }  
\end{array} \right. $}
\\
\fwdangvel_{\goal}(\pos, \ort) &\!:=\! 
\scalebox{0.92}{$\left \{ 
 \begin{array}{@{}c@{}l@{}}
\anggain \fwdheadingerror_{\goal}\!(\pos, \ort) \!+\! \frac{\lingain}{2}\sin(2\fwdheadingerror_{\goal}\!(\pos, \ort)\!),  & \text{ if } |\fwdheadingerror_{\goal}\!(\pos, \ort)|\! \leq \! \tfrac{\pi}{2} \\
\anggain \fwdheadingerror_{\goal}(\pos, \ort), & \text{ otherwise }  
\end{array} \right. $}
\end{align}
\end{subequations}
}%
where $\lingain>0$ and $\anggain>0$ are fixed positive control gains, and the angular heading error $\fwdheadingerror_{\goal}(\pos, \ort)$ of the forward-moving unicycle state $(\pos, \ort)$ relative to $\goal$ is defined as 
{\small
\begin{equation}\label{eq.Forward_Angular_Heading_Error}
\!\fwdheadingerror_{\goal}(\pos, \ort) \!:= \arctantwo\plist{\!\nvecTsmall{\ort}\!\! (\goal\!-\!\pos) ,  \! \ovecTsmall{\ort}\!\! (\goal\!-\!\pos)\!\!}.\!\!
\end{equation}
}%
Here, $\arctantwo(y,x)$ is the 2-argument inverse tangent function that returns the counterclockwise angle in radians in $[-\pi, \pi)$ from the horizontal axis to the ray starting from the origin to the point $(x,y)$. 
To resolve indeterminacy, we set $\fwdheadingerror_{\goal}(\pos, \ort) = 0$ for $\pos = \goal$.
 
Since the forward unicycle control $\fwdctrl_{\goal}$ in \refeq{eq.Forward_Unicycle_Control} generates the same closed-loop unicycle motion as the  bidirectional unicycle control $\ctrl_{\goal}$ in \refeq{eq.Bidirectional_Unicycle_Control} when $|\fwdheadingerror_{\goal}(\pos, \ort)| \leq \frac{\pi}{2}$, the ball-shaped, cone-shaped, and diamond-shaped motion predictions of the forward unicycle control $\fwdctrl_{\goal}$ can be performed using the associated motion prediction methods of the bidirectional unicycle control $\ctrl_{\goal}$  as
{\small
\begin{subequations} \label{eq.Forward_Unicycle_Motion_Predictions}
\begin{align}
\motionset_{\fwdctrl_{\goal}, \ball}(\pos, \ort)&:= \motionset_{\ctrl_{\goal}, \ball}(\pos, \ort)
\\
\!\!\motionset_{\fwdctrl_{\goal}, \cone}(\pos, \ort)&:= \!\left \{ 
\begin{array}{@{}c@{}l@{}}
\motionset_{\ctrl_{\goal}, \cone}(\pos, \ort), & \text{ if } |\fwdheadingerror_{\goal}\!(\pos, \ort)| \leq \frac{\pi}{2} \\
\motionset_{\ctrl_{\goal}, \ball}(\pos, \ort), & \text{ otherwise}
\end{array}
\right .  \!\! 
\\
\!\!\motionset_{\fwdctrl_{\goal}, \dshape}(\pos, \ort)&:= \! \left \{ 
\begin{array}{@{}l@{}}
\motionset_{\ctrl_{\goal}, \dshape}(\pos, \ort),  \text{ if } |\fwdheadingerror_{\goal}\!(\pos, \ort)| \leq \frac{\pi}{2} \\
\conv(\goal, \overline{\xpos}_{\goal}(\pos), \overline{\xpos}_{\goal}^{r}\!(\pos)\!),  \text{ otherwise}
\end{array}
\right .  \!\! 
\end{align} 
\end{subequations}
}%
where the motion prediction sets $\motionset_{\ctrl_{\goal},\ball}$, $\motionset_{\ctrl_{\goal},\cone}$, and $\motionset_{\ctrl_{\goal},\dshape}$ are defined as in \refeq{eq.Circular_Motion_Prediction}, \refeq{eq.Conic_Motion_Prediction}, \refeq{eq.Diamond-Shaped_Motion_Prediction}, respectively, and the current and final heading line intersection $\overline{\xpos}_{\goal}(\pos)$  and its reflection $\overline{\xpos}_{\goal}^{r}(\pos)$ are given for $\fwdheadingerror_{\goal}\!(\pos, \ort) = \pm \frac{\pi}{2}$ by\footnote{For $|\headingerror_{\goal}\!(\pos, \ort)| = \frac{\pi}{2}$, the final orientation and the total turning satisfy $|\headingerror_{\goal}(\pos, \ort^*_{\goal}(\pos, \ort))| = \frac{\lingain}{2\anggain} \Si(\pi)$ and $|\totalturning_{\goal}(\pos, \ort)| = \tfrac{\pi}{2} + \frac{\lingain}{2\anggain} \Si(\pi)$. }
{\small
\begin{subequations}
\begin{align}
\overline{\xpos}_{\goal}(\pos)& \scalebox{0.95}{$:= \pos \!+\! \scalebox{0.9}{$\dfrac{\sin(\frac{\lingain}{2\anggain} \Si(\pi)\!)}{\sin(\tfrac{\pi}{2}+\frac{\lingain}{2\anggain} \Si(\pi)\!)}$}  \Rmat(-\tfrac{\pi}{2}) (\goal\! - \!\pos)$}\! 
\\
\overline{\xpos}^{r}_{\goal}(\pos)& \scalebox{0.95}{$:= \pos \!+\! \scalebox{0.9}{$\dfrac{\sin(\frac{\lingain}{2\anggain} \Si(\pi)\!)}{\sin(\tfrac{\pi}{2}+\frac{\lingain}{2\anggain} \Si(\pi)\!)}$}  \Rmat(+\tfrac{\pi}{2}) (\goal\! - \!\pos)$}.\! 
\end{align}
\end{subequations}
}%
Hence, the motion prediction sets $\motionset_{\fwdctrl_{\goal},\ball}$, $\motionset_{\fwdctrl_{\goal},\cone}$, and $\motionset_{\fwdctrl_{\goal},\dshape}$ of the forward unicycle control $\fwdctrl_{\goal}$ inherit key characteristic properties such as positive inclusion (\refprop{prop.Positive_Inclusion}), radial decay (\refprop{prop.Radial_Decay}), Lipschitz continuous distance (\refprop{prop.Motion_Prediction_Distance}), and inclusion order (\refprop{prop.Inclusion_Order}).    
Moreover, similar to \refprop{prop.Total_Turning_Effort}, the total turning effort of the forward unicycle control $\fwdctrl_{\goal}$ can be determined as 
{\small
\begin{align}
\fwdtotalturning_{\goal}(\pos, \ort)\!\!:=\! \fwdheadingerror_{\goal}(\pos, \ort) \!+\! \frac{\lingain}{2\anggain} 
\scalebox{0.81}{$
\left \{ \!
\begin{array}{@{}c@{}l@{}}
\Si(-\pi), & \text{ if } \fwdheadingerror_{\goal}(\pos, \ort) \!<\! -\tfrac{\pi}{2} \\
\Si(2\fwdheadingerror_{\goal}(\pos, \ort)\!), & \text{ if }  -\!\tfrac{\pi}{2}\! \leq\! \fwdheadingerror_{\goal}(\pos, \ort) \!\leq\! \tfrac{\pi}{2} \\
\Si(\pi), & \text{ if } \fwdheadingerror_{\goal}(\pos, \ort) \!>\! \tfrac{\pi}{2}. 
\end{array}
\right.
$}
\end{align}  
}%
  
\subsection{Backward Unicycle Control and Motion Prediction}

Similar to \refeq{eq.Forward_Unicycle_Control}, the backward-directional unicycle control $\bckctrl_{\goal}(\pos, \ort):= (\bcklinvel_{\goal}(\pos, \ort), \bckangvel_{\goal}(\pos, \ort))$ can be designed for a unicycle robot with the backward movement constraint, as
{\small
\begin{subequations}\label{eq.Backward_Unicycle_Control}
\begin{align}
\bcklinvel_{\goal}(\pos, \ort) &:= \min \plist{0, \lingain \ovecTsmall{\ort}\! \!(\goal\!-\!\pos)\!\!} \nonumber\\
& =  \scalebox{0.92}{$\left \{ 
 \begin{array}{@{}c@{}l@{}}
\lingain \ovecTsmall{\ort}\! \!(\goal\!-\!\pos),  & \text{ if } |\bckheadingerror_{\goal}\!(\pos, \ort)| \! \leq \! \tfrac{\pi}{2} \\
0, & \text{ otherwise }  
\end{array} \right. $}
\\
\bckangvel_{\goal}(\pos, \ort) &\!:=\! 
\scalebox{0.92}{$\left \{ 
 \begin{array}{@{}c@{}l@{}}
\anggain \bckheadingerror_{\goal}\!(\pos, \ort) \!+\! \frac{\lingain}{2}\sin(2\bckheadingerror_{\goal}\!(\pos, \ort)\!),  & \text{ if } |\bckheadingerror_{\goal}\!(\pos, \ort)|\! \leq \! \tfrac{\pi}{2} \\
\anggain \bckheadingerror_{\goal}(\pos, \ort), & \text{ otherwise }  
\end{array} \right. $}
\end{align}
\end{subequations}
}%
where $\lingain>0$ and $\anggain>0$ are fixed positive control gains and the angular heading error $\bckheadingerror_{\goal}(\pos, \ort)$ of the backward-moving unicycle state $(\pos, \ort)$ relative to $\goal$ is defined as 
{\small
\begin{equation}\label{eq.Forward_Angular_Heading_Error}
\!\!\bckheadingerror_{\goal}(\pos, \ort) \!:= \arctantwo\plist{\!-\nvecTsmall{\ort}\!\! (\goal\!-\!\pos) ,  \! -\ovecTsmall{\ort}\!\! (\goal\!-\!\pos)\!\!}.\!\!\!
\end{equation}
}%
Hence, as in \refeq{eq.Forward_Unicycle_Motion_Predictions}, the circular, conic, and diamond-shaped feedback motion prediction sets for the backward unicycle control $\bckctrl_{\goal}$ can be constructed as
{\small 
\begin{subequations}
\begin{align}
\motionset_{\bckctrl_{\goal}, \ball}(\pos, \ort)&:= \motionset_{\ctrl_{\goal}, \ball}(\pos, \ort)
\\
\!\!\motionset_{\bckctrl_{\goal}, \cone}(\pos, \ort)&:= \!\left \{ 
\begin{array}{@{}c@{}l@{}}
\motionset_{\ctrl_{\goal}, \cone}(\pos, \ort), & \text{ if } |\bckheadingerror_{\goal}\!(\pos, \ort)| \leq \frac{\pi}{2} \\
\motionset_{\ctrl_{\goal}, \ball}(\pos, \ort), & \text{ otherwise}
\end{array}
\right .  \!\! 
\\
\!\!\motionset_{\bckctrl_{\goal}, \dshape}(\pos, \ort)&:= \! \left \{ 
\begin{array}{@{}l@{}}
\motionset_{\ctrl_{\goal}, \dshape}(\pos, \ort),  \text{ if } |\bckheadingerror_{\goal}\!(\pos, \ort)| \leq \frac{\pi}{2} \\
\conv(\goal, \overline{\xpos}_{\goal}(\pos), \overline{\xpos}_{\goal}^{r}\!(\pos)\!),  \text{ otherwise}
\end{array}
\right .  \!\! 
\end{align} 
\end{subequations}
}%
which come with the positive inclusion (\refprop{prop.Positive_Inclusion}), radial decay (\refprop{prop.Radial_Decay}), Lipschitz continuous distance (\refprop{prop.Motion_Prediction_Distance}), and inclusion order (\refprop{prop.Inclusion_Order}) properties. 
Finally, as in \refprop{prop.Total_Turning_Effort}, the total turning effort of the backward unicycle control $\bckctrl_{\goal}$ is given by
{\small
\begin{align}
\bcktotalturning_{\goal}(\pos, \ort)\!\!:=\! \bckheadingerror_{\goal}(\pos, \ort) \!+\! \frac{\lingain}{2\anggain} 
\scalebox{0.81}{$
\left \{ \!
\begin{array}{@{}c@{}l@{}}
\Si(-\pi), & \text{ if } \bckheadingerror_{\goal}(\pos, \ort) \!<\! -\tfrac{\pi}{2} \\
\Si(2\bckheadingerror_{\goal}(\pos, \ort)\!), & \text{ if }  -\!\tfrac{\pi}{2}\! \leq\! \bckheadingerror_{\goal}(\pos, \ort) \!\leq\! \tfrac{\pi}{2} \\
\Si(\pi), & \text{ if } \bckheadingerror_{\goal}(\pos, \ort) \!>\! \tfrac{\pi}{2}. 
\end{array}
\right.
$}
\end{align}  
}%

\end{document}